\documentclass{article}
\usepackage[T1]{fontenc}
\usepackage[english]{babel}
\usepackage[margin=1.1in]{geometry}
\usepackage[utf8]{inputenc}
\usepackage{lmodern}
\usepackage{microtype}

\usepackage{amsmath,amssymb,amsfonts,amsthm}
\usepackage{mathtools}
\usepackage{mathrsfs} 
\usepackage[square,comma,numbers]{natbib}
\usepackage[hidelinks]{hyperref}
\usepackage{graphicx}
\usepackage{enumerate}
\usepackage[font=small]{caption}
\usepackage{subcaption}
\usepackage{algorithm}
\usepackage[noend]{algpseudocode}
\usepackage[dvipsnames]{xcolor}
\usepackage{tikz}
\usetikzlibrary{arrows.meta}

\numberwithin{equation}{section}
\newtheorem{Thm}{Theorem}[section]
\newtheorem*{Thm*}{Theorem}
\newtheorem{Lem}[Thm]{Lemma}
\newtheorem{Asm}[Thm]{Assumption}
\newtheorem{Rmk}[Thm]{Remark} 
\newtheorem{Ex}[Thm]{Example} 

\newtheorem{Prop}[Thm]{Proposition}

\def\cB{{\mathcal B}}
\def\cC{{\mathcal C}}
\def\cE{{\mathcal E}}
\def\cF{{\mathcal F}}
\def\cG{{\mathcal G}}
\def\cP{{\mathcal P}}
\def\cX{{\mathcal X}}
\def\cU{{\mathcal U}}
\def\cS{{\mathcal S}}
\def\cT{{\mathcal T}}

\def\fd{\mathfrak{d}}
\def\fR{\mathfrak{R}}
\def\fg{\mathfrak{g}}
\def\fm{\mathfrak{m}}
\def\fz{\mathfrak{z}}

\def\E{\mathbb{E}}
\def\P{\mathbb{P}}
\def\R{\mathbb{R}}

\def\sF{{\mathscr{F}}}
\def\sG{\mathscr{G}}
\def\eps{\epsilon}

\newcommand{\one}{\mathbf 1}
\newcommand{\Proj}{\mathrm{Proj}}
\newcommand{\ip}[2]{\langle #1,#2\rangle}
\renewcommand{\d}{\mathrm{d}}

\usepackage{fancyhdr}
\title{Decentralized Bandit Learning of Markov Perfect Equilibria\footnote{Research 
is partially supported by the National Science Foundation grant DMS 2406762.}}
\author
{Felix H\"ofer\footnote{Booth School of Business, The University of Chicago, Chicago, IL, 60637, USA, email: 
{\tt felix.hoefer@chicagobooth.edu}}
\and H.~Mete Soner\footnote{ 
Department of Operations Research and Financial
Engineering, Princeton University, Princeton, NJ, 08540, USA, email: 
{\tt soner@princeton.edu}}}

\date{\today}

\begin{document}
\maketitle

\begin{abstract}
\noindent
We study decentralized bandit learning in discounted infinite-horizon, general-sum Markov games with continuous actions. Agents observe the public state, but only their own realized rewards, and use simultaneous perturbation stochastic approximation to update their actions. Under strong monotonicity of the negative running-reward gradients and sufficiently small time steps, we prove that the fully decentralized procedure converges almost surely to the unique stationary Markov perfect equilibrium. We obtain an $O(n^{-1/2})$ convergence rate and present a natural extension to episodic finite-horizon learning.
\end{abstract}
\vspace{3pt}

\vspace{4pt}


\section{Introduction}

We consider a discounted infinite-horizon general-sum Markov game. Agents construct strategies based only on observations of the current state, their own realized action, and their own realized reward. They do not know the objectives of the other agents and need not even be aware of their existence. We develop a fully decentralized, payoff-based learning procedure and prove its convergence to the unique stationary Markov perfect equilibrium (MPE).

Learning Nash equilibria is a classical and difficult problem. Among the classical
multi-agent reinforcement-learning (MARL) algorithms are Littman's minimax-Q \cite{Litt},
Hu and Wellman's Nash-Q \cite{HW}, and Bowling and
Veloso's variable-learning-rate method
\cite{BV}. In general games,
however, even natural learning dynamics may not converge, and global 
convergence guarantees
rely on additional structure, such as potential games or weakly acyclic games. 
In this paper, we establish conditions under which dynamic equilibria can be learned 
in both stationary infinite-horizon and finite-horizon Markov games.
Our structural condition is of a control-theoretic kind:
we impose strong monotonicity of the players' negative running reward gradients, which is equivalent to strong concavity of running rewards if they are independent of other players' actions. In the finite-horizon setting, 
H\"{o}fer, Soner, and Y{\i}lmaz \cite{HSY-MPE} showed that this condition ensures 
the uniqueness of the dynamic MPE for sufficiently small time steps, and H\"ofer \emph{et al.}~\cite{HLSY} 
use it to prove geometric convergence of model-based iterative schemes.

In our algorithm, agents update their strategies by stochastic gradient ascent, 
with the required gradient estimated by simultaneous perturbation 
stochastic approximation \cite{Spall,Spall1} from their private observations. 
 A perturbed action, the realized reward, and an estimate of the discounted 
 continuation value yield a one-point estimate of each agent's 
 pseudo-gradient. The continuation values are unknown and must be estimated as learning proceeds, 
 adding a difficulty specific to the dynamic setting.
Under strong monotonicity and a small time-step 
condition, the coupled policy and value updates converge almost surely to the 
unique stationary MPE.

Zhang and Mu~\cite{zhang2026stochastic} also study payoff-based equilibrium 
learning in discounted games by playing repeated terminating trials to estimate policy gradients. In contrast, our algorithm uses a single uninterrupted trajectory of the original game to learn both policies and continuation values simultaneously, and our policy and value updates do not require separate timescales. The recent paper of Zhang and Mazumdar~\cite{zhang2026provably} establishes convergence of learning dynamics in discounted Markov games to stationary risk-averse quantal-response equilibria, but requires that agents observe opponents' actions. By comparison, we learn ordinary MPE in continuous-action games from private bandit feedback alone.

The closest static counterparts are the papers of
Bravo, Leslie, and Mertikopoulos \cite{BLM} and 
Tatarenko and Kamgarpour \cite{TG}. Both study continuous-action monotone games
in which agents learn from realized payoffs: \cite{TG} uses a
distributed payoff-based procedure with coupling constraints, while
\cite{BLM} analyzes no-regret bandit learning. With slightly differing
assumptions and guarantees, they provide natural static benchmarks
for our work. In the strongly monotone static setting, Drusvyatskiy,
Fazel, and Ratliff \cite{DML} improve the rate obtained by \cite{BLM}, 
see also the convergence results of Ba \emph{et al.}~\cite{ba2025doubly}; 
we obtain the same order of mean-squared error in our dynamic setting.

Our dynamic setting introduces several additional layers. The target
is an MPE rather than the Nash equilibrium of a static game, and the
agents' objectives depend on continuation values which must be learned
along with their policies. At the same time, the observations come
from a single trajectory: actions affect future states, and the states
visited change as the policies evolve. The analysis must therefore
control the policy and value errors together, while  preserving the needed
monotone structure.

Mertikopoulos and Zhou \cite{MZ} provide a complementary static
comparison. They work under variational stability without requiring concavity,
but use noisy first-order feedback 
(i.e., an oracle providing noisy observations of the payoff gradient) 
rather than bandit feedback. They
prove global almost-sure convergence under global stability and
local convergence with high probability under local stability. In our setting,
the quantitative bound arising from the concavity assumption gives
global convergence from arbitrary initialization and helps control
the error from the evolving continuation estimates.

For Markov games, Arslan and Y\"uksel \cite{AY} and 
Yongacoglu, Arslan, and Y\"uksel \cite{YAY2,YAY1}
give decentralized $Q$-learning results for finite-action weakly acyclic and 
common-interest settings, see Ozdaglar \emph{et al.}~\cite{OSZ} for an overview. 
Their methods use exploration phases and comparisons among finitely many 
actions. Sayin \emph{et al.}~\cite{SZLBO} study decentralized learning in 
two-agent zero-sum games, while Maheshwari \emph{et al.}~\cite{MWPS}  
exploit the Markov-potential structure. Chen and Li \cite{CL} estimate gradients 
of an overall policy-performance criterion from trajectories in general-sum games. 
In our continuous-action general-sum setting, each agent estimates
the gradient of its Bellman operator with respect to its own action
from a private one-step observation.
The error due to the unknown continuation value is treated as a vanishing 
stochastic-approximation bias and controlled through a contraction 
of the coupled policy and value errors. 

Our updates and numerical 
sequences are synchronized across agents.  
Unsynchronized decentralized $Q$-learning has recently been studied 
 \cite{YAY1}, building on the classical asynchronous
stochastic approximation and $Q$-learning of Tsitsiklis \cite{T94}.
Motivated by \cite{YAY1}, a natural extension of our approach would 
allow agent-dependent clocks and step sizes.

We consider both finite-horizon and infinite-horizon discounted criteria
and obtain similar results in the two settings. To avoid repetition,
we first describe the game and algorithm in the infinite-horizon
setting. Subsection \ref{ss:episodic} introduces episodic
finite-horizon learning, and Section \ref{sec:episodic} establishes
its convergence.

\subsection{Markov game and the Nash system}
\label{ss:problem}

There are $N\ge1$ agents, indexed by $i\in[N]:=\{1,\ldots,N\}$, and decision periods $t=0,1,\ldots$, 
uniformly spaced with a common time step $h>0$. Agent $i$ has a nonempty compact convex action set $\cU^i\subset\R^{d_i}$ 
with nonempty interior. By an appropriate translation of the action set, we assume without 
loss of generality that $0$ lies in its interior.

Set $d=\sum_{i=1}^N d_i$ and
$$
\cU:=\prod_{i=1}^N \cU^i\subset\R^d,
\qquad a=(a^1,\ldots,a^N)\in\cU,
$$
and use $a^{-i}$ for the other agents' actions.  
Our analysis extends directly to state-dependent action spaces, which arise in some applications, but for the ease of exposition we refrain from this generality.

The \emph{observed joint} state process $(X_t)_{t\ge0}$ takes values in a finite set $\cX$. 
Let $\cP(\cX)$ be the set of all probabilities on $\cX$, and
given $x\in\cX$ and $a\in\cU$, the next state is drawn from a controlled transition kernel
$$
\P(\cdot \mid x,a)\in\cP(\cX).
$$
After the joint action $A_t$ is played at time $t$, 
agent $i$ receives a running reward  $\fR_{t+1}^i$,  together with the successor state $X_{t+1}$. 
We assume that 
$\E [\fR_{t+1}^i \mid X_t=x, A_t=a ] = h r^i(x,a)$ 
for some $r^i: \cX \times \cU \to \R$.
Agent $i$ uses a fixed and known discount factor $\beta^i\in(0,1)$. For a stationary Markov feedback profile $\alpha:\cX\to\cU$, define the expected reward 
\begin{equation*}
J^i(x;\alpha)
:=\E_x^\alpha\Big[\sum_{t=0}^{\infty}(\beta^i)^t\fR_{t+1}^i\Big]
=\E_x^\alpha\Big[\sum_{t=0}^{\infty}(\beta^i)^t hr^i(X_t,\alpha(X_t))\Big].
\end{equation*}
The objective is to \emph{maximize} this expected reward.
The expected reward function $r^i$ and the transition 
kernel $\P(\cdot \mid x,a)$ are unknown to the agents.  

For  $W=(W^i)_{i=1}^N:\cX\to\R^N$, define the \emph{Bellman operator} 
and its \emph{pseudo-gradient} by
$$
Q^i[W](x,a):=hr^i(x,a)+\beta^i \mathsf P [W^i](x,a) \in \R,
\qquad
\Phi[W](x,a):=\bigl(\nabla_{a^i}Q^i[W](x,a)\bigr)_{i=1}^N\in\R^d,
$$
where
$$
\mathsf P [f](x,a):=\sum_{y\in\cX}f(y)\P( y\mid x,a).
$$

A pair $(\alpha^\star,V^\star)$ is a (stationary) \emph{Markov perfect equilibrium} if $\alpha^\star$ is a stationary feedback profile and, for every $i\in[N]$, $x\in\cX$, and unilateral deviation $\pi^i$,
\begin{equation*}
V^{i,\star}(x)=J^i(x;\alpha^\star)
\ge J^i(x;(\pi^i,\alpha^{\star,-i})).
\end{equation*}
 Under the differentiability and monotonicity conditions imposed below, 
 the equilibrium $(\alpha^\star,V^\star)$
 is characterized by the stationary \emph{Nash system}:
\begin{equation}
V^{i,\star}(x)=Q^i[V^\star](x,\alpha^\star(x)),
\qquad \forall  x \in \cX,
\label{eq:equilibrium-value}
\end{equation}
\begin{equation}
\ip{\alpha-\alpha^\star(x)}{\Phi[V^\star](x,\alpha^\star(x))}
\le0,\qquad \forall \alpha\in\cU,\, x \in \cX.
\label{eq:equilibrium-VI}
\end{equation}
We note that in our setting, the variational inequality is equivalent to
$$
\alpha^{i,\star}(x)\in \ \mathop{\rm arg\,max}_{u^i\in \cU^i}
Q^i[V^\star](x,(u^i,\alpha^{\star,-i}(x))).
$$

\subsection{Algorithm}
\label{ss:algo}

The learning procedure operates along one uninterrupted trajectory. Learning iterations are indexed by $n\ge0$. At iteration $n$, the current state is $X_n$ and the transition produces updated tables with index $n+1$. The initial tables $\widehat\alpha^0$ and $\widehat V^0$ are given and deterministic. The initial state $X_0$ is a given arbitrary random variable in $\cX$. {Agent} $i$ maintains a stationary policy table $\widehat\alpha^{i,n}(x)$ and a value table $\widehat V^{i,n}(x)$. The algorithm uses sequences $(\eta_n,\nu_n,\lambda_n,\delta_n)_{n\ge0}$ which are assumed to satisfy Assumption \ref{asm:observations} below.

\paragraph{Step 1. Play one action.}
At iteration $n\ge0$, agent $i$ plays the perturbed action
\begin{equation}
A_n^i:=(1-\lambda_n)\widehat\alpha^{i,n}(X_n)
       +\delta_n\Delta_n^i\in\cU^i,
\label{eq:played-action}
\end{equation}
where the private perturbations $\Delta_n^i$ are conditionally
independent across agents. 

\paragraph{Step 2. Form the estimates.}
After observing its private running reward $\fR_{n+1}^i$ and the successor state $X_{n+1}$
generated by the joint action $A_n=(A_n^1,\ldots,A_n^N)$, agent $i$ forms
\begin{equation}
\widetilde V^{i,n+1}
   :=\fR_{n+1}^i+\beta^i\widehat V^{i,n}(X_{n+1}),
\qquad
\widehat D^{i,n+1}
   :=\frac{d_i}{\delta_n}\widetilde V^{i,n+1} \Delta_n^i.
\label{eq:bellman-sample}
\end{equation}
Thus, $\widetilde V^{i,n+1}$ is the sample used to update the value function estimate, 
and $\widehat D^{i,n+1}$ is a one-point estimate of agent $i$’s 
component of the pseudo-gradient used in the stochastic gradient ascent step.

The random perturbation gives a one-point estimate of the action
gradient \cite{Spall,Spall1}, with the factor $d_i$ 
accounting for the action dimension.

\paragraph{Step 3. Update the visited table entries.}
Set
\begin{align}
\widehat\alpha^{i,n+1}(X_n)
=\Proj_{\cU^i}\bigl(\widehat\alpha^{i,n}(X_n)+\eta_n\widehat D^{i,n+1}\bigr),
\quad
\widehat V^{i,n+1}(X_n)
=(1-\nu_n)\widehat V^{i,n}(X_n)+\nu_n\widetilde V^{i,n+1}.
\label{eq:value-update}
\end{align}
Unvisited entries remain unchanged.

\medskip
\noindent\emph{Information structure:}
\label{rmk:information}
The joint state is publicly observed. Agent $i$ knows its own action set, dimension, 
interior radius, discount factor $\beta^i$, the time step $h$, the prescribed numerical sequences, 
and the perturbations. At iteration $n$, it uses only
$$
(X_n,\Delta_n^i,A_n^i,\fR_{n+1}^i,X_{n+1}),
$$
together with its own tables. It does not know the transition kernel or its derivatives, its own reward function beyond the observed samples, or the other agents' reward functions, realized rewards, actions, perturbations, or tables. Knowledge of the other agents' discount factors is unnecessary, although these design choices may also be specified publicly. There is no communication between agents.

\begin{Thm*}[Main result, informal]
Under the regularity, strong-monotonicity, and
small-time-step assumptions stated below, there exists a unique Markov
perfect equilibrium pair $(\alpha^\star,V^\star)$. Furthermore, under
the learning assumptions stated below, $(\widehat\alpha^n,\widehat V^n)$
almost surely converges to  $(\alpha^\star,V^\star)$.
\end{Thm*}
Existence and uniqueness are established in Theorem \ref{thm:mpe},
and almost-sure convergence in Theorem \ref{thm:convergence}.
Theorem \ref{thm:rate} also gives an $O(n^{-1/2})$
mean-squared error bound.

\paragraph{Organization of the remaining paper.} Section~\ref{sec:finite} states our assumptions and main results.
Section~\ref{sec:strategic-service} illustrates the framework with a
strategic queueing model. Section~\ref{sec:lemmas} collects preliminary
lemmas. Sections~\ref{sec:mpe}, \ref{secLconvergence},
and~\ref{sec:rate} prove existence and uniqueness of the MPE,
almost-sure convergence of the algorithm, and the mean-squared
convergence rate, respectively. Section~\ref{sec:episodic} explains how our algorithm is applied to episodic finite-horizon learning.

\paragraph{Notation.}  $B(x,\rho)$ denotes the closed ball centered at $x$ with
radius $\rho>0$. Throughout, we identify 
every agent-indexed family $f^1,\ldots,f^N$ of
functions with $f=(f^i)_{i=1}^N$.
Euclidean norms and inner products are denoted by $\|\cdot\|$ and $\ip{\cdot}{\cdot}$. 
For agent-indexed scalar functions $W=(W^1,\ldots,W^N)$ and 
feedback policies $\alpha=(\alpha^1,\ldots,\alpha^N)$, set
$$
\|W\|_\infty:=\sup_{x\in\cX}\max_{i\in[N]}|W^i(x)|,
\qquad
\|\alpha\|_{\infty,2}:=\sup_{x\in\cX}\|\alpha(x)\|.
$$

\section{Assumptions and results}
\label{sec:finite}
In this section, we first state our technical assumptions and then 
state all results.
\subsection{Assumptions on the game}
Define the \emph{reward pseudo-gradient} by
$$
F(x,a):=\bigl(\nabla_{a^i}r^i(x,a)\bigr)_{i=1}^N.
$$
Recall that $r^i$ is the expected unknown reward function of agent $i$;
see \eqref{eq:observed-reward} below.
We regard $\P(\cdot\mid x,a)$ and its action derivatives as vectors in $\R^{\cX}$.

\begin{Asm}[Rewards]
\label{asm:rewards}
For every $i\in[N]$ and $x\in\cX$, the map $a\mapsto r^i(x,a)$ 
is continuously differentiable, and there are constants 
$C_r,L_r,L_F<\infty$ and $\mu>0$ such that, uniformly over $i,x,a,b$,
\begin{align}
|r^i(x,a)|&\le C_r,
\qquad \|\nabla_ar^i(x,a)\|\le L_r,
\label{eq:reward-bounds}\\
\|F(x,a)-F(x,b)\|&\le L_F\|a-b\|,
\notag\\
\ip{F(x,a)-F(x,b)}{a-b}&\le-\mu\|a-b\|^2.
\label{eq:running-monotonicity}
\end{align}
Realized rewards also satisfy $\|\fR_n^i \| \le h C_r$ for all $n$ and $i$.
\end{Asm}

Note that \eqref{eq:running-monotonicity} is equivalent to $-F$ being strongly $\mu$-monotone.
Since the number of players is finite and each $\beta^i<1$,
$\beta_\star:=\max_i\beta^i<1$. Thus, for every admissible strategy
profile $\alpha$, 
$$
|J^i(x;\alpha)|\le M^i_\star:=\frac{hC_r}{1-\beta^i} \le
M_\star:=\max_i M^i_\star =\frac{hC_r}{1-\beta_\star} <\infty.
$$
In light of this observation, we introduce
$$
\cB:=\{W:\cX\to\R^N\text{ such that } |W^i(x)|\le M^i_\star\text{ for all }i,x\}.
$$
\begin{Asm}[Transition kernel]
\label{asm:transition}
For every $x\in\cX$, the probability vector map $a\mapsto \P(\cdot\mid x,a)$ 
is twice continuously differentiable, and there are $c_1,c_2<\infty$ 
such that, for all $i,j\in[N]$, $x\in\cX$, $a\in \cU$, $v^i\in\R^{d_i},\ v^j\in\R^{d_j}$,
\begin{align}
\label{eq:transition}
\|D_i\P(\cdot\mid x,a)[v^i]\|_1
\le hc_1\|v^i\|,
\qquad
\|D_{ji}^2\P(\cdot\mid x,a)[v^j,v^i]\|_1
\le h^2c_2\|v^j\|\|v^i\|.
\end{align}
\end{Asm}

Let $\mu$ be as in \eqref{eq:running-monotonicity} and set
\begin{equation}
\mu_h:=h (\mu- h \beta_\star N M_\star c_2),\qquad
L_h:=h L_r+M_\star L_\star, \qquad 
L_\star:=h\, \beta_\star \sqrt N \, c_1.
\label{eq:discounted-constants}
\end{equation}

\begin{Asm}[Time Step]
\label{asm:gap}
The time step $h>0$ is sufficiently small
so that 
\begin{equation}
\mu_h>0,
\qquad
\kappa_h:=\beta_\star+\frac{ L_h L_\star}{\mu_h} <1.
\label{eq:contraction}
\end{equation}
The contraction condition $\kappa_h<1$ is \emph{not} required for episodic learning.
\end{Asm}

If we treat all constants as independent of $h$, 
then $\mu_h= h (\mu -O(h))$ and $\kappa_h=\beta_\star+O(h)$ as $h\downarrow0$. 
Hence, the above condition 
requires the time step to be sufficiently small.
Equivalently, when the model is obtained by time-discretizing 
a continuous-time problem, the continuous-time discount rate and sampling frequency must be sufficiently large.
The latter allows the agents to react promptly to changes in the common state. 

\begin{Ex}[Affine jump kernel]
Suppose that $\cX$ is finite, and for $y\ne x$,
$$
q(y\mid x,a)=q^0(y\mid x)+\sum_{i=1}^N B^i(y\mid x)\cdot a^i\ge0
\qquad(a\in\cU),
$$
for some kernels $q^0, B^i$.
Set $\overline q:=\max_{x,a}\sum_{y\ne x}q(y\mid x,a)$ and $B:=\max_{i,x}\sum_{y\ne x}\|B^i(y\mid x)\|$. If $h\overline q\le1$, define
$$
\P(y\mid x,a)=
\begin{cases}
hq(y\mid x,a),&y\ne x,\\
1-h\displaystyle\sum_{z\ne x}q(z\mid x,a),&y=x.
\end{cases}
$$
The affine action dependence gives Assumption~\ref{asm:transition}
with $c_1=2B$ and $c_2=0$, and hence $\mu_h=h\mu>0$.
The discounted contraction condition $\kappa_h<1$ is still a separate requirement.
\end{Ex}

\subsection{Assumptions on the algorithm}

For the analysis, define a sequence of sigma-algebras by
$$
\cF_0:=\sigma(X_0),\qquad
\cG_n:=\cF_n\vee\sigma(\Delta_n),\qquad
\cF_{n+1}:=\cG_n\vee\sigma(\fR_{n+1},X_{n+1}),\qquad
\qquad n\ge0,
$$
where $\Delta_n=(\Delta_n^i)_{i=1}^N$, $A_n=(A_n^i)_{i=1}^N$, and 
$\fR_{n+1}=(\fR_{n+1}^i)_{i=1}^N$ denote the joint vectors. 
The current state and the tables $\widehat\alpha^n,\widehat V^n$ are $\cF_n$-measurable 
while the joint action $A_n$ is $\cG_n$-measurable by \eqref{eq:played-action}. 
The value samples, gradient estimates, and updated tables are $\cF_{n+1}$-measurable.

Choose $\varrho_\star>0$ such that
\begin{equation}
\label{eqn:rhostar}
B(0,\varrho_\star)\subset \cU^i
\qquad\text{for every }i\in[N].
\end{equation}
Recall that Assumption \ref{asm:rewards} requires that 
all rewards are bounded: $|\fR_n^i|\le h  C_r$.

\begin{Asm}[Observations and numerical sequences]
\label{asm:observations}
Almost surely, for every step $n\ge0$,
state $x\in\cX$, and agent $i\in[N]$,
\begin{equation}
\P(X_{n+1}=x\mid\cG_n)=\P(x\mid X_n,A_n),
\qquad
\E[\fR_{n+1}^i\mid\cG_n]=hr^i(X_n,A_n).
\label{eq:observed-reward}
\end{equation}
Conditional on $\cF_n$, the perturbations are independent across agents and satisfy
\begin{equation}
\E[\Delta_n^i\mid\cF_n]=0,
\quad
\E[\Delta_n^i(\Delta_n^i)^\top\mid\cF_n]=\frac1{d_i}I_{d_i},
\quad
\|\Delta_n^i\|\le1.
\label{eq:direction-moments}
\end{equation}
All scalar sequences appearing in the algorithm are deterministic and satisfy
\begin{gather}
\eta_n>0,\qquad 0<\nu_n, \lambda_n  \le 1,\qquad 0<\delta_n\le\varrho_\star\lambda_n,
\label{eq:basic-sequences}\\
\delta_n\to0,\qquad\lambda_n\to0,\qquad
\sum_n\eta_n=\infty,\qquad
\sum_n\frac{\eta_n^2}{\delta_n^2}<\infty,
\label{eq:policy-sequences}\\
\sum_n\nu_n=\infty,\qquad\sum_n\nu_n^2<\infty.
\label{eq:value-sequences}
\end{gather}
\end{Asm}

The above assumption guarantees
that actions are admissible.
Indeed, for $\|\Delta^i\|\le1$, \eqref{eq:basic-sequences}
implies that $\|(\delta_n/\lambda_n)\Delta^i\|\le \varrho_\star$.
Hence by \eqref{eqn:rhostar}, $(\delta_n/\lambda_n)\Delta^i \in \cU^i$.
Then by the convexity of $\cU^i$,
for any $u^i\in \cU^i$, 
\begin{equation*}
(1-\lambda_n)u^i+\delta_n\Delta^i
=(1-\lambda_n)u^i+\lambda_n \left((\delta_n/\lambda_n) \Delta^i\right)\in \cU^i.
\end{equation*}
Also, note that the reward realization and the successor state need not be independent. 

\begin{Asm}[Exploration]
\label{asm:accessibility}
The state space is finite, and there is $\zeta>0$ such that
\begin{equation*}
\P(y\mid x,a)\ge\zeta,
\qquad x,y\in\cX,\quad a\in\cU.
\end{equation*}
\end{Asm}

\begin{Rmk}\label{rmk:acessibility-weak}
Assumption \ref{asm:accessibility} is only one of many possible sufficient assumptions. For example, for a  feedback profile
$\alpha:\cX\to\cU$, define the transition matrix $P_\alpha(x,y):=\P(y\mid x,\alpha(x))$, $x,y\in\cX$. 
Without proof, we note that one may replace Assumption \ref{asm:accessibility} by the requirement that there exist an integer $m\ge1$ and a constant $\zeta>0$
such that
$$
\bigl(P_{\alpha_1}P_{\alpha_2}\cdots
P_{\alpha_{m}}\bigr)(x,y)\ge\zeta
$$
for every $x,y\in\cX$ and every sequence of feedback
profiles $\alpha_1,\ldots,\alpha_{m}:\cX\to\cU$.
\end{Rmk}

\subsection{Main results}

In this section, we state the main results.  Their proofs are given 
in the subsequent sections.

\begin{Thm}[Existence and uniqueness of MPE]
\label{thm:mpe}
Under Assumptions~\ref{asm:rewards}, \ref{asm:transition}, and~\ref{asm:gap}, 
there exists a unique stationary MPE $(\alpha^\star,V^\star)$. 
It is the unique solution of \eqref{eq:equilibrium-value}--\eqref{eq:equilibrium-VI}, and $V^\star\in\cB$, i.e.,
$$
|V^{i,\star}(x)|\le M^i_\star=\frac{hC_r}{1-\beta^i}
\qquad\forall\, x \in \cX, i\in[N].
$$
\end{Thm}

\begin{Rmk}
Uniqueness also holds in the broader class of time-dependent Markov policies. Indeed, let
$(\alpha_t,V_t)_{t\ge0}$ be such an equilibrium, with $V_t$ denoting
its continuation value discounted relative to time $t\ge0$. The proof of Theorem~\ref{thm:mpe} yields a strict contraction $\cT_h:\cB\to\cB$ so that the values satisfy $V_t=\cT_hV_{t+1}$. Repeated application of this map yields, for every $t\ge0$ and $m\ge1$,
$$
\|V_t-V^\star\|_\infty
\le\kappa_h^m\|V_{t+m}-V^\star\|_\infty
\le2M_\star\kappa_h^m.
$$
Letting $m\to\infty$ with $t$ fixed yields $V_t=V^\star$ and hence $\alpha_t=\alpha^\star$. 
\end{Rmk}

We initialize the algorithm deterministically
satisfying
$\widehat\alpha^{i,0}(x)\in \cU^i$,
$|\widehat V^{i,0}(x)|\le\overline M$
with some constant $\overline M\ge M_\star= hC_r/(1-\beta_\star)$.

\begin{Thm}[Almost-sure convergence]
\label{thm:convergence}
Under Assumptions~\ref{asm:rewards}, \ref{asm:transition}, \ref{asm:gap}, \ref{asm:observations}, \ref{asm:accessibility}, 
and the above initialization, the algorithm \eqref{eq:played-action}--\eqref{eq:value-update} satisfies
\begin{equation}
\|\widehat\alpha^n-\alpha^\star\|_{\infty,2}
+\|\widehat V^n-V^\star\|_\infty\longrightarrow0
\qquad\text{almost surely}.
\label{eq:main-convergence}
\end{equation}
Moreover, $\|A_n-\alpha^\star(X_n)\|\to0$ almost surely.
\end{Thm}

To state our convergence rate, we define
$$
\mathsf A^n:=\sum_{x\in\cX}\E\bigl[\|\widehat\alpha^n(x)-\alpha^\star(x)\|^2\bigr],
\qquad
\mathsf V^n:=\sum_{x\in\cX}
\E\bigl[\|\widehat V^{n}(x)-V^{\star}(x)\|^2\bigr],
$$
and make the dependencies of the numerical sequences on $n$ explicit:
for all sufficiently large $n$, 
\begin{equation}
\eta_n=\frac{{\omega}}{\mu_h n},
\qquad \nu_n=\frac{{\omega}}{n},
\qquad \delta_n=\frac{\delta_\circ}{n^{1/4}},
\qquad \lambda_n=\frac{\lambda_\circ}{n^{1/4}},
\label{eq:rate-schedules}
\end{equation}
for some positive constants
$\omega,\delta_\circ,\lambda_\circ$
satisfying $\delta_\circ\le\varrho_\star\lambda_\circ$,
where $\mu_h$ is as in \eqref{eq:discounted-constants}  and $\varrho_\star$ as in \eqref{eqn:rhostar}. 

\begin{Thm}[Convergence rate]
\label{thm:rate}
Suppose that all hypotheses of Theorem~\ref{thm:convergence} 
and the numerical schedules~\eqref{eq:rate-schedules} hold. There exists
$\omega_\star=\omega_\star(\mu_h,C_r,L_h,L_\star,\beta_\star,N,\zeta)>0$
such that for every $\omega > \omega_\star$, there is $K<\infty$ such that
\begin{equation}
\mathsf A^n+\mathsf V^n\le\frac K{\sqrt{n}},\qquad n\ge1.
\label{eq:mean-square-rate}
\end{equation}
In particular,
\begin{equation}
\E\big[\|\widehat\alpha^n-\alpha^\star\|_{\infty,2}^2
+\|\widehat V^n-V^\star\|_\infty^2\big]
\le\frac K{\sqrt{n}},
\qquad
\E\bigl[\|A_n-\alpha^\star(X_n)\|^2\bigr]\le\frac K{\sqrt{n}}.
\label{eq:max-rate}
\end{equation}
\end{Thm}

\subsection{Episodic learning}
\label{ss:episodic}
In this section, we explain how our algorithm naturally carries over to finite-horizon problems.  
Consider a model with a terminal time horizon $T$, with decision periods $t=0,\ldots,T-1$. 
Expected rewards $r_t^i$, transitions $\P_t$, and feedbacks may depend on $t$
and there is an expected terminal reward $g^i:\cX \to \R$. 
For $x\in\cX$ and feedback controls $\alpha^i_t:\cX\to\cU^i$, $t<T$, 
consider the expected payoff
\begin{equation*}
J_t^i(x;\alpha):=\E_{t,x}^\alpha\Big[
\sum_{s=t}^{T-1}hr_s^i(X_s,\alpha_s(X_s))
+g^i(X_T)\Big].
\end{equation*}
Define the Bellman operator and its pseudo-gradient  by
$$
Q_t^i[W](x,a):=hr_t^i(x,a)+\mathsf P_t[W^i](x,a),
\qquad
\Phi_t[W](x,a):=\bigl(\nabla_{a^i}Q_t^i[W](x,a)\bigr)_{i=1}^N.
$$
The Nash system reads for $i\in[N]$, $t<T$, and $x\in\cX$,
\begin{equation}
\begin{split}
V_T^{i,\star}(x)&=g^i(x),\\
\ip{u-\alpha_t^\star(x)}{\Phi_t[V_{t+1}^\star](x,\alpha_t^\star(x))}&\le0,
\qquad \forall u\in\cU,\\
V_t^{i,\star}(x)&=Q_t^i[V_{t+1}^\star](x,\alpha_t^\star(x)).
\end{split}
\label{eq:episodic-nash-system}
\end{equation}

\paragraph{Episodic algorithm.}
Episodes are indexed by $n\ge0$. The environment produces a reset state $X_0^n$, 
and the policy and value tables with superscript $n$ are held fixed during episode $n$. 
The observations in that episode produce tables with superscript $n+1$. Agent $i$ plays
$$
A_t^{i,n}:=(1-\lambda_n)\widehat\alpha_t^{i,n}(X^n_t)+\delta_n\Delta_t^{i,n}.
$$
It receives rewards $\fR^{i,n}_{t+1}$ (resp.~$\fg^{i,n}$ at terminal time), 
observes a new state $X^n_{t+1}$, and forms the estimates
\begin{equation*}
\widetilde V_t^{i,n+1}:=
\begin{cases}
\fR_{t+1}^{i,n}+\widehat V_{t+1}^{i,n}(X_{t+1}^n),&t<T,\\
\fg^{i,n},&t=T,
\end{cases}
\qquad
\widehat D_t^{i,n+1}:=\frac{d_i}{\delta_n} \Delta_t^{i,n}\widetilde V_t^{i,n+1}
\quad(t<T).
\end{equation*}
After one entire play of the game, it updates the entries visited in that episode by
\begin{align*}
\widehat\alpha_t^{i,n+1}(X_t^n)
&=\Proj_{\cU^i}\bigl(\widehat\alpha_t^{i,n}(X_t^n)+\eta_n\widehat D_t^{i,n+1}\bigr),
\qquad t<T,
\\
\widehat V_t^{i,n+1}(X_t^n)
&=(1-\nu_n)\widehat V_t^{i,n}(X_t^n)+\nu_n\widetilde V_t^{i,n+1},
\qquad t\le T.
\end{align*}
Unvisited entries remain unchanged. This is the same one-point ascent and value averaging as before, with the next-time table replacing the stationary continuation table.

We prove an analogous convergence result in Section~\ref{sec:episodic} below.

\section{Illustrative example}
\label{sec:strategic-service}

In this section, we consider a discounted dynamic service game inspired by the
strategic queueing model of Zhong, Gopalakrishnan, and Ward
\cite{ZhongGopalakrishnanWard2025}. Here, the agents are service providers
who choose their processing rates. For example, they may operate computing
servers that process jobs arriving at a shared queue. Providers receive
contractual payments for the processing capacity they supply, pay operating
costs, and receive standby payments when idle. All rewards are net monetary
payments observed by the providers. Each provider's processing rate affects the shared queue and thereby the assignment of future jobs. These interactions give rise to a
non-cooperative dynamic game.

\vspace{-10pt}
\paragraph{State and actions.}
There are $N\ge2$ service providers (agents), each operating one server, and a finite total capacity of jobs $k\ge N$, including
jobs in service. With this convention, at most $k-N$ jobs can wait. A job is
assigned uniformly among idle servers, waits if all servers are busy and space remains, and otherwise leaves. Waiting jobs enter service in order of arrival. Decision periods are indexed by $t=0,1,\ldots$, with time step $h>0$.
The public state is $X_t=(\ell_t,S_t)$, where $\ell_t$ counts waiting
jobs and $S_t\subseteq[N]$ records the busy servers. Formally, we let the state space be 
$$
\cX=(\{0\}\times2^{[N]})
\cup(\{1,\ldots,k-N\}\times\{[N]\}),
$$
where $2^{[N]}$ denotes the set of all subsets of $[N]$.
When not idle, after observing $X_t$, agent $i$ chooses its
service rate $\alpha_t^i\in[\underline\alpha,\overline\alpha]$, where
$0<\underline\alpha<\overline\alpha<\infty$. 

\vspace{-10pt}
\paragraph{Transitions.}
Let $\lambda>0$ be the arrival rate of jobs, and
require $h(\lambda+N\overline \alpha)<1$. At each update, for a given joint action  
$a=(a^1,\ldots,a^N)$ and $S \subset [N]$, the following transitions can happen:
$$
\begin{array}{c|c|c}
\text{transition} & \text{probability} & \text{conditions} \\ \hline
(0,S)\to(0,S\cup\{i\}) & h\lambda/(N-|S|) & |S|<N, \ i\notin S \\
(0,S)\to(0,S\setminus\{i\}) & h a^i & i\in S \\
(\ell,[N])\to(\ell+1,[N]) & h\lambda & 0\le\ell<k-N \\
(\ell,[N])\to(\ell-1,[N]) & h\sum_{i=1}^Na^i
 & 1\le\ell\le k-N
\end{array}
$$
The first row models an arrival assigned uniformly at random among
idle agents. The second row records a service completion when the
queue is empty, after which the agent becomes idle. The third row
models an arrival joining the queue when all agents are busy and
space remains. In the last row, a service completion occurs and the
first waiting job immediately replaces the departing job
at the same agent; the queue length decreases by one while all agents
remain busy. The remaining probability is assigned to no change. An arrival to a full system is rejected and leaves the state unchanged. We illustrate the transitions 
for a problem with $N=2$ and  $k=4$
in Figure \ref{fig:service-transitions} below.

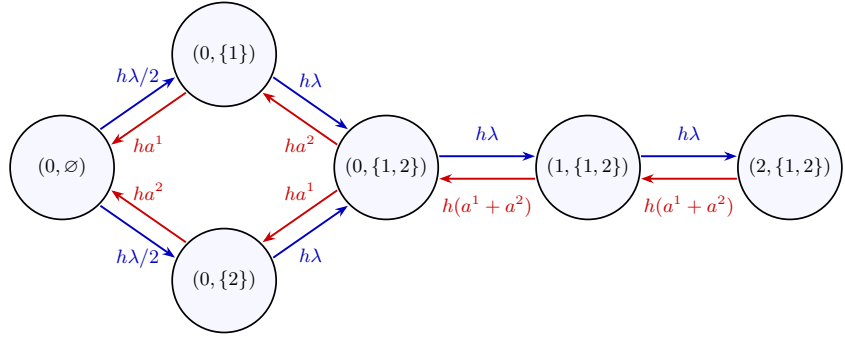
\begin{figure}[h]
\centering
\resizebox{0.7\linewidth}{!}{
\begin{tikzpicture}[
  font=\small,
  >={Stealth[length=2.1mm,width=1.5mm]},
  qstate/.style={circle,draw=black,fill=blue!3,
    line width=0.8pt,minimum size=17mm,inner sep=1pt,
    outer sep=1pt,text=black},
  arrival/.style={->,draw=blue!75!black,text=blue!75!black,
    line width=0.9pt},
  completion/.style={->,draw=red!80!black,text=red!80!black,
    line width=0.9pt}
]
% State: (number of waiting jobs, set of busy servers).
\node[qstate] (empty) at (0,0)       {$(0,\varnothing)$};
\node[qstate] (one)   at (2.65,1.85) {$(0,\{1\})$};
\node[qstate] (two)   at (2.65,-1.85){$(0,\{2\})$};
\node[qstate] (busy)  at (5.30,0)    {$(0,\{1,2\})$};
\node[qstate] (wait1) at (8.60,0)    {$(1,\{1,2\})$};
\node[qstate] (wait2) at (11.90,0)   {$(2,\{1,2\})$};

% Empty system <-> only server 1 busy.
\draw[arrival] (empty.45) --
  node[midway,above=4pt] {$h\lambda/2$} (one.205);
\draw[completion] (one.225) --
  node[midway,below=4pt] {$ha^1$} (empty.25);

% Empty system <-> only server 2 busy.
\draw[arrival] (empty.315) --
  node[midway,below=4pt] {$h\lambda/2$} (two.155);
\draw[completion] (two.135) --
  node[midway,above=4pt] {$ha^2$} (empty.335);

% Only server 1 busy <-> both busy: server 2 completes on return.
\draw[arrival] (one.335) --
  node[midway,above=4pt] {$h\lambda$} (busy.135);
\draw[completion] (busy.155) --
  node[midway,below=4pt] {$ha^2$} (one.315);

% Only server 2 busy <-> both busy: server 1 completes on return.
\draw[arrival] (two.25) --
  node[midway,below=4pt] {$h\lambda$} (busy.225);
\draw[completion] (busy.205) --
  node[midway,above=4pt] {$ha^1$} (two.45);

% Both servers stay busy when a waiting job enters service.
\draw[arrival] (busy.12) --
  node[midway,above=4pt] {$h\lambda$} (wait1.168);
\draw[completion] (wait1.192) --
  node[midway,below=4pt] {$h(a^1+a^2)$} (busy.348);
\draw[arrival] (wait1.12) --
  node[midway,above=4pt] {$h\lambda$} (wait2.168);
\draw[completion] (wait2.192) --
  node[midway,below=4pt] {$h(a^1+a^2)$} (wait1.348);
\end{tikzpicture}
}
\caption{State transitions for two servers ($N=2$) and total capacity
$k=4$. Each node $(\ell,S)$ records the number of waiting jobs and
the set of busy servers. Blue arrows represent arrivals and red arrows
service completions. Self-transitions, including rejected
arrivals at full capacity, are omitted.}
\label{fig:service-transitions}
\end{figure}

\paragraph{Rewards.}
We use a capacity-based payment contract. 
When busy, provider $i\in[N]$
receives a payment at rate $p^i a^i$, where $p^i\ge0$, and  incurs 
a quadratic operating cost with parameter $c_E>0$, representing
electricity, rented computing capacity or other expenses. When idle,
the provider receives a standby payment at rate $v^i\ge0$ for remaining
available. Payments are based on the capacity supplied during the decision
period, rather than on the number of service completions. The net payment is therefore
\begin{equation}
r^i(x,a) =p^i a^i-c_E(a^i)^2, \ \text{when}\  i\in S,
\qquad\text{and}\qquad
r^i(x,a) =v^i , \ \text{when}\  i\notin S,
\label{eq:service-reward}
\end{equation}
where $x=(\ell,S)$, $a^i\in[\underline{\alpha},\overline{\alpha}]$ for $i\in S$, and $a^i=0$ for $i\notin S$.
At each learning iteration, provider $i$ observes its net monetary reward
$\fR_{n+1}^i:=h r^i(X_n,A_n).$
The provider need not know the payment or operating-cost
functions; the learning algorithm uses only the observed net payment.
Agent $i$ seeks to maximize
$$
\E_x^\alpha\Big[\sum_{t=0}^{\infty}(\beta^i)^t
 h r^i(X_t,\alpha(X_t))\Big],
\qquad \beta^i\in(0,1).
$$

To apply our learning algorithm, set
$\alpha_{\mathrm c}:=(\underline\alpha+\overline\alpha)/2$ and $
\varrho:=(\overline\alpha-\underline\alpha)/2$. We apply our algorithm of Section~\ref{ss:algo} to the centered actions
$u^i:=a^i-\alpha_{\mathrm c}\in\cU^i:=[-\varrho,\varrho]$ when $i\in S$, and set $\cU^i:=\{0\}$ when $i\notin S$,
with $d_i=1$ and $\varrho_\star=\varrho$.
Equivalently, writing $\widehat\alpha^{i,n}$ for the
service-rate table, replace \eqref{eq:played-action} by
$$
A_n^i
:=(1-\lambda_n)\widehat\alpha^{i,n}(X_n)
+\lambda_n\alpha_{\mathrm c}+\delta_n\Delta_n^i
\in[\underline\alpha,\overline\alpha],
$$
when it is busy, and $A_n^i=0$ otherwise. Perturbations and policy
updates are performed only for busy providers; value updates are performed
for every provider. All other steps remain unchanged.

We assume that $h>0$ is small enough so that
\begin{equation}
h(\lambda+N\overline\alpha)<1\qquad \text{and}\qquad 
\kappa_h:=\beta_\star+
\frac{\beta_\star\sqrt N\,h}{c_E}
\Big(L_r+\frac{2\beta_\star\sqrt N\,hC_r}{1-\beta_\star}\Big)<1,
\label{eq:service-contraction}
\end{equation}
with
\begin{equation}
C_r:=\max_i(p^i\overline\alpha+v^i)+c_E\overline\alpha^2,
\qquad
L_r:=\max_i p^i+2c_E\overline\alpha.
\label{eq:sercive-small-time}
\end{equation}

\begin{Prop}[Equilibrium and learning in the service model]
\label{prop:service}
Assume \eqref{eq:sercive-small-time} and \eqref{eq:service-contraction} hold. 
Then, the service game has a unique stationary MPE $(\alpha^\star,V^\star)$, with
$|V^{i,\star}(x)|\le hC_r/(1-\beta^i)$.
Furthermore, under Assumption \ref{asm:observations} 
and  the initialization specified before Theorem~\ref{thm:convergence}, the service-rate and value tables of the modified algorithm described above converge almost surely to the unique MPE as in
\eqref{eq:main-convergence}.
\end{Prop}

\begin{proof} As the action spaces naturally depend on the state,
this model  fits our setting with minor modifications. 
 For each agent, the centered action space
is $[-\varrho,\varrho]$ when it is busy and $\{0\}$ when it is idle.
We omitted state-dependent action spaces from the general formulation
to simplify the exposition. Here, action perturbations and policy
updates are performed only for busy agents; idle agents use the
prescribed zero service rate, while value updates are performed
at every state. With this modification, the arguments establishing
equilibrium and convergence apply as follows.

Each busy provider has centered action set $[-\varrho,\varrho]$,
which contains zero in its interior. Idle providers have no choice as
their actions are fixed at zero. The reward bounds hold with the
displayed $C_r,L_r$. For feasible physical service rates $a,b$
at state $x$, we have
$$
F_i(x,a)=p^i\one_{\{i\in S\}}-2c_Ea^i,
\qquad
\ip{F(x,a)-F(x,b)}{a-b}=-2c_E\|a-b\|^2.
$$
Idle coordinates satisfy $a^i=b^i=0$, so this calculation gives the
required bounds on the active coordinates. 
Thus $\mu=L_F=2c_E$, and it is straightforward to 
see that Assumption \ref{asm:transition} holds with
$c_1=2$, $c_2=0$, $\mu_h=2hc_E$, and $\kappa_h$ as in \eqref{eq:service-contraction}.

To verify the exploration assumption, put
$$
\epsilon:=\min\{h\lambda/N,\ h\underline\alpha,
\ 1-h(\lambda+N\overline\alpha)\}>0.
$$
From any state, at most $k$ service completions lead to the empty state;
from there, at most $k$ prescribed arrivals and assignments lead to any
target state. Every state has a self-transition of probability at least $\epsilon$, 
uniformly over admissible actions, including the
perturbed actions used by the algorithm. 
Thus, any such path of length $r\le2k$ can be extended to exactly
$2k$ steps by appending $2k-r$ self-transitions at its terminal state. 
The condition in Remark~\ref{rmk:acessibility-weak} therefore holds with
$m=2k$ and $\zeta=\epsilon^{2k}$, for every sequence of feedback profiles.
The claim follows from Remark~\ref{rmk:acessibility-weak} and the
arguments of Theorems~\ref{thm:mpe} and \ref{thm:convergence}, applied
with the state-dependent action spaces described above.
\end{proof}

\section{Preliminaries}
\label{sec:lemmas}

We first collect a few preliminary lemmas.
Let $\mu_h, L_h,L_\star$ be as in \eqref{eq:discounted-constants},
$L_r,L_F$ be as in Assumption~\ref{asm:rewards},
and $c_1,c_2$ be as in Assumption~\ref{asm:transition}.

\begin{Lem}[One-step games]
\label{lem:bellman}
Suppose that Assumptions~\ref{asm:rewards} and~\ref{asm:transition} hold
and  $W,\widetilde W$ are bounded tables with $\|W\|_\infty \le \fm$.  Then
\begin{align}
\|\Phi[W](x,a)-\Phi[W](x,b)\| &\le L_\Phi(\fm)\|a-b\|,
\label{eq:Phi-action-lip}\\
\|\Phi[W](x,a)-\Phi[\widetilde W](x,a)\| &\le L_\star\|W-\widetilde W\|_\infty,
\label{eq:Phi-W-lip}\\
|Q^i[W](x,a)-Q^i[W](x,b)| &\le L_Q(\fm)\|a-b\|,
\label{eq:Q-action-lip}
\end{align}
where 
$$
L_\Phi(\fm):=hL_F+h^2\, \beta_\star N c_2 \fm,
\qquad
L_Q(\fm):=h\bigl(L_r+\beta_\star\sqrt N c_1\fm \bigr).
$$
In particular, $L_Q(M_\star)=L_h$, and if  $\|W\|_\infty\le M_\star$, then
\begin{equation}
\ip{\Phi[W](x,a)-\Phi[W](x,b)}{a-b}
\le-\mu_h\|a-b\|^2.
\label{eq:bellman-monotonicity}
\end{equation}
\end{Lem}

A straightforward proof is given in Appendix~\ref{app:bellman}.
Note that \eqref{eq:bellman-monotonicity} states that the negative
pseudo-gradient is strongly monotone, a property inherited
from the strong monotonicity assumption \eqref{eq:running-monotonicity} on $-F$.

Next, we prove a gradient estimate that will be used to control
the one-point estimate of agent $i$'s component of the pseudo-gradient
in \eqref{eq:bellman-sample}.  Below $\sF$ is a sigma-field, 
$W: \Omega \times \cX\to\R^N$ is a bounded $\sF$-measurable function, 
and $\lambda,\delta$ are uniformly positive $\sF$-measurable satisfying  $0<\lambda\le1$ and 
$0<\delta\le\varrho_\star\lambda$ a.s. Conditional on $\sF$, perturbations 
$\Delta=(\Delta^1,\ldots,\Delta^N)$ are independent and 
satisfy the moment conditions in \eqref{eq:direction-moments}. 
For some $\sF$-measurable $x \in \cX$, $a \in \cU$,  there is a  bounded random vector  
$\widetilde V = (\widetilde V^1,\ldots, \widetilde V^N)$  such that, for $i \in [N]$,
$$
\E[\widetilde V^i\mid\sG]=Q^i[W](x,(1-\lambda)a+\delta\Delta),
\qquad\text{where}\qquad
\sG:=\sF\vee\sigma(\Delta^1,\ldots,\Delta^N).
$$
Define the one-point estimate $\widehat D^i=(d_i/\delta) \widetilde V^i \Delta^i$,
$B:=\E[\widehat D\mid\sF]-\Phi[W](x,a)$, and
$\xi:=\widehat D-\E[\widehat D\mid\sF]$.
It is clear that $B$ is $\sF$-measurable, $\E[\xi\mid\sF]=0$, and 
$\widehat D=\Phi[W](x,a)+B+\xi$.
\begin{Lem}[Gradient estimate]
\label{lem:gradient}
Let $\sF, W, \widetilde V, \widehat D, B, \xi, \lambda, \delta$ be as above.
Then there is a deterministic constant $K$, depending on the fixed model constants, 
action sets, dimensions, and the uniform bounds of $W$ and $\widetilde V$, such that
\begin{equation}
\|B\|\le K(\delta+\lambda),
\qquad
\E[\|\xi\|^2\mid\sF]\le\frac{K}{\delta^2}.
\label{eq:one-point-bounds}
\end{equation}
\end{Lem}

The short proof is given in Appendix~\ref{app:gradient}. 
Next, we recall a classical result due to Robbins and Siegmund.

\begin{Lem}[Robbins and Siegmund \cite{RobbinsSiegmund1971}]
\label{lem:robbins-siegmund}
Let $(\sF_n)_{n\ge0}$ be a filtration, and let
$(Z_n,v_n,w_n)_{n\ge0}$ be non-negative $(\sF_n)$-adapted processes such
that each $Z_n$ is integrable and
$$ 
\E[Z_{n+1}\mid\sF_n]\le Z_n-v_n+w_n
\qquad\text{almost surely.}
$$ 
If $\sum_nw_n<\infty$ a.s., then almost surely $Z_n$ converges to
a finite random variable and $\sum_nv_n<\infty$.
\end{Lem}

The preceding lemma yields the following central stochastic approximation result. 
Below, we consider 
$$
\text{a nonempty, compact, convex set}\ \cC\subset\R^p,
$$
a filtration $(\sF_n)_{n\ge0}$, and random functions
$\Psi_n:\Omega\times\cC\to\R^p$ that are
$(\sF_n\otimes\cB(\cC),\cB(\R^p))$-measurable,
where $\cB(\cdot)$ denotes the Borel sigma-algebra.
The $\sF_n$-measurable sequences 
$k_n\in\R$ and $b_n,\xi_n\in\R^p$ satisfy
$$
k_n\ge0,\qquad
\xi_n\in L^2(\sF_n;\R^p),\qquad
\E[\xi_{n+1}\mid\sF_n]=0,
$$
and
\begin{equation}
\sum_n k_n=\infty,
\qquad
\sum_n k_n^2
\left(1+\E[\|\xi_{n+1}\|^2\mid\sF_n]\right)<\infty
\quad\text{a.s.}
\label{eq:sa-steps}
\end{equation}
Starting from an $\sF_0$-measurable random variable $\fz_0\in\cC$,
the sequence $\fz_n$ is recursively defined by
\begin{equation*}
\fz_{n+1}
:=\Proj_\cC\left( \fz_n-k_n\bigl(\Psi_n(\fz_n)+b_n+\xi_{n+1}\bigr) \right),
\qquad n\ge0.
\end{equation*}

\begin{Prop}[Projected stochastic approximation]
\label{prop:sa}
Let $\cC,\sF_n,\Psi_n,b_n,k_n,\xi_n,\fz_n$ be given as above
satisfying \eqref{eq:sa-steps}.
Suppose that there exist $\fz^\star\in\cC$, $b, c_\star \ge 0$,
$\mu>0$ such that, for every $n\ge0$ and $\fz\in\cC$,
\begin{equation}
\|\Psi_n(\fz)\|\le c_\star,
\qquad
\ip{\fz-\fz^\star}{\Psi_n(\fz)} \ge \mu \|\fz-\fz^\star\|^2,
\qquad
\limsup_{n\to\infty}\|b_n\|\le b
\qquad\text{a.s.}
\label{eq:sa-coercivity}
\end{equation}
Then
$\limsup_{n\to\infty}\|\fz_n-\fz^\star\|  \le b/\mu$ almost surely.
\end{Prop}

\begin{proof}
We fix a constant $0<\eps\le 1$ and suppress its 
dependence in the notation throughout the proof.
Set
$b_\star:= (b + 2 \eps)/\mu$, $\cS:=\cC\cap B(\fz^\star,b_\star)$,
$d_n:=\|\fz_n-\fz^\star\|$, and
$\mathfrak{d}_n:= \|\fz_n-y_n\|$,
where $y_n$ is given by
$$
y_n = \gamma_n \, \fz^\star + (1-\gamma_n)\, \fz_n,
\qquad \text{with}\qquad 
 \gamma_n:= \frac{1}{d_n}\, \left(d_n- \min\{ b_\star, d_n\}\right)\ \text{when}\ d_n>0,
$$
and $y_n=\fz_n=\fz^\star$ when $d_n=0$.
It is clear that  $y_n\in \cS$ by the convexity of $\cC$ and
\begin{equation}
\fd_n:= \|\fz_n-y_n\| = \mathrm{dist}(\fz_n,\cS) 
= \max\{d_n -b_\star,0\} \le \|\fz_n-z\|, 
\qquad \forall\  z\in \cS.
\label{eq:Sk}
\end{equation}
By definition, $d_n=\|\fz_n -\fz^\star\| =
\fd_n +\|y_n -\fz^\star \|$.  Since $y_n \in \cS$, $\|y_n -\fz^\star\| \le b_\star$.
Hence, 
\begin{equation}
\label{eq:dDR}
\fd_n \le d_n = \|\fz_n -\fz^\star\| \le \fd_n + b_\star, \qquad
d_n \gamma_n = \fd_n.
\end{equation}

Suppose that $\|b_n\|\le b+ \eps$.  We claim that
\begin{equation}
\ip{\fz_n-y_n}{\Psi_n(\fz_n)+b_n}\ge \eps  \fd_n = \eps \|\fz_n-y_n\|.
\label{eq:sa-drift}
\end{equation}
Indeed, this is immediate when $\fd_n=\|\fz_n-y_n\|=0$. Otherwise,
$$
\begin{aligned}
\ip{\fz_n-y_n}{\Psi_n(\fz_n)+b_n} & 
= \gamma_n \ip{\fz_n-\fz^\star}{\Psi_n(\fz_n)+b_n}
= \frac{\fd_n}{d_n}\, \left( \ip{\fz_n-\fz^\star}{\Psi_n(\fz_n)} 
+  \ip{\fz_n-\fz^\star}{b_n}\right)\\
 & \ge \frac{\fd_n}{d_n}\, \left( \mu  \|\fz_n-\fz^\star\|^2 
 -  \|\fz_n-\fz^\star\| \|b_n\|\right)
 = \fd_n\, ( \mu d_n - \|b_n\| ) \ge  \fd_n\, (\mu d_n - b-\eps).
\end{aligned}
$$
Since $\fd_n=\max\{d_n -b_\star,0\}>0$,  $\mu d_n >\mu b_\star =b+2\eps$. Hence
$\mu d_n - b- \eps  > \eps$, proving the claim.

For $m\ge1$, define a stopping time
$\sigma_m:=\inf\{n\ge m:\|b_n\| > b + \eps\}$, 
where by convention it is equal to
infinity if the set is empty. Then
$\|b_n \| \le b+ \eps$ for every $m \le n < \sigma_m$.
Suppressing the dependence on $m$, set 
$Z_n :=(\fd_{n\land \sigma_m})^2$.
We claim that,  for $n\ge m$, 
\begin{equation}
\label{eq:Zn-claim}
\E[Z_{n+1}\mid\sF_n] \le Z_n
-2 \eps k_n \fd_n\one_{\{n<\sigma_m\}}+C_\star\, k_n^2
\left(1+\E[\|\xi_{n+1}\|^2\mid\sF_n]\right)
\one_{\{n<\sigma_m\}},
\end{equation}
for some constant $C_\star$, independent of $n$. 

We first complete the proof of the Proposition and then prove the claim.
In view of \eqref{eq:sa-steps}, we  have
$$
\sum_{n}k_n^2 \left(1+\E[\|\xi_{n+1}\|^2\mid\sF_n]\right)
\one_{\{m \le n<\sigma_m\}}<\infty.
$$
Hence we may now apply Lemma~\ref{lem:robbins-siegmund} 
to conclude that  almost surely $Z_n$ is convergent
and  that $\sum_{n} k_n \fd_n\one_{\{m \le n<\sigma_m\}} <\infty$.
In particular, on the set $\Omega_m:=\{\sigma_m=\infty\}$, we have 
$\sum_{n} k_n \fd_n <\infty$.  
As by hypothesis \eqref{eq:sa-steps}, $\sum_n k_n =\infty$, the limit of
$\fd_n$ must be zero on $\Omega_m$.   Hence by \eqref{eq:dDR},
$$
\limsup_{n \to \infty} \|\fz_n-\fz^\star\|
\le  \lim_{n \to \infty} \fd_n + b_\star =b_\star ,
 \qquad \text{on}\ \Omega_m,
$$
for every $\eps > 0$ and $m \ge1$.
Since $\limsup_{n \to \infty} \|b_n\| \le b$, the definitions of $\sigma_m$
and $\Omega_m$ directly imply that
$\P(\bigcup_m \Omega_m)=1$.  Hence, for every $\eps > 0$,
$$ 
\limsup_{n\to\infty}\|\fz_n-\fz^\star\| \le b_\star =(b+2\eps)/\mu,
\qquad\text{a.s.}
$$ 
A direct limiting argument as $\eps$ tends to zero completes the proof of the Proposition.

We continue with the proof of the claimed inequality \eqref{eq:Zn-claim}.
Let  $z_n:=\fz_n-k_n(\Psi_n(\fz_n)+b_n+\xi_{n+1})$.  Since $y_n\in \cS \subset \cC$, \eqref{eq:Sk} 
implies that $\fd_{n+1} \le \|\fz_{n+1}-y_n\|$, and
$y_n=\Proj_\cC(y_n)$. Hence
$$ 
(\fd_{n+1})^2
\le\|\fz_{n+1}-y_n\|^2 =\|\Proj_\cC(z_n)-\Proj_\cC(y_n)\|^2 \le \|z_n-y_n\|^2,
$$ 
where the last inequality uses the fact that projections are non-expansive.  
Using the definitions, we directly calculate that
\begin{align*}
\|z_n-y_n\|^2 &= \|\fz_n-y_n-k_n(\Psi_n(\fz_n)+b_n+\xi_{n+1}) \|^2\\
&= (\fd_n)^2- 2 k_n 
\ip{\fz_n-y_n}{\Psi_n(\fz_n)+b_n+\xi_{n+1}}
+ k_n^2 \|\Psi_n(\fz_n)+b_n+\xi_{n+1}\|^2.
\end{align*}
Since $\E[\xi_{n+1}\mid\sF_n]=0$ and $(k_n,b_n,\Psi_n)$ are $\sF_n$-measurable,
$\E[k_n \ip{\fz_n-y_n}{\xi_{n+1}} \mid\sF_n] =0$.
By \eqref{eq:sa-drift}
$$
\E[k_n \ip{\fz_n-y_n}{\Psi_n(\fz_n)+b_n+\xi_{n+1}}\mid\sF_n]
= k_n \E[\ip{\fz_n-y_n}{\Psi_n(\fz_n)+b_n}\mid\sF_n] \ge 
\eps k_n \fd_n.
$$
Moreover, for $m \le n<\sigma_m$, we have $\|b_n\|\le b+ \eps \le  b+ 1$
and, by hypothesis, $\|\Psi_n(\fz_n)\|\le c_\star$. Hence,
$$
\E[\|\Psi_n(\fz_n)+b_n+\xi_{n+1}\|^2\mid\sF_n]
=\|\Psi_n(\fz_n)+b_n\|^2
 +\E[\|\xi_{n+1}\|^2\mid\sF_n] \le C_\star \left(1+\E[\|\xi_{n+1}\|^2\mid\sF_n]\right),
$$
where $C_\star:=\max\{1,(c_\star+b+1)^2\}$.

Combining all the above inequalities, we arrive at
\begin{align*}
\E[(\fd_{n+1})^2 \mid\sF_n]
\le& (\fd_{n})^2
-2 \eps k_n\fd_n +C_\star k_n^2 \left(1+\E[\|\xi_{n+1}\|^2\mid\sF_n]\right),
\qquad \text{for} \ m \le n < \sigma_m.
\end{align*}
Since $Z_n :=(\fd_{n\land \sigma_m})^2$ is constant
for $n \ge \sigma_m$, the above inequality implies \eqref{eq:Zn-claim}.
\end{proof}

\section{Existence and uniqueness of the MPE}
\label{sec:mpe}

\begin{proof}[Proof of Theorem~\ref{thm:mpe}]
For $W\in\cB$, let $S_h(W)(x)\in\cU$ be the unique solution to
$$
\langle u-S_h(W)(x),\, \Phi[W](x,S_h(W)(x))\rangle\le0,
\qquad u\in\cU.
$$
Existence follows from Brouwer's theorem applied to the continuous map 
$a\mapsto\Proj_{\cU}(a+\Phi[W](x,a))$. Its fixed points are precisely the displayed 
variational-inequality solutions. Indeed, since $\cU$ is a convex set and  $a = \Proj_\cU(z)$, 
we have  $\ip{u-a}{z-a} \le 0$, for all $u\in\cU$.
Applying this with $z=a+\Phi[W](x,a)$ shows that  
$$
a = \Proj_\cU(a+\Phi[W](x,a))\qquad \Leftrightarrow \qquad 
\ip{u-a}{\Phi[W](x,a)} \le 0\quad \forall u\in\cU.
$$
Then, Lemma~\ref{lem:bellman} shows that $-\Phi[W]$ is strongly monotone, 
so that such solutions are indeed unique.

For $W,\widetilde W\in\cB$, write $a=S_h(W)(x)$, $b=S_h(\widetilde W)(x)$. 
Testing the variational inequalities against one another, adding and subtracting $\Phi[W](x,b)$, 
and using \eqref{eq:bellman-monotonicity} and \eqref{eq:Phi-W-lip}, leads to 
\begin{align*}
0&\le \ip{a-b}{\Phi[W](x,a)-\Phi[\widetilde W](x,b)}
= \ip{a-b}{[\Phi[W](x,a)-\Phi[W](x,b)]+[ \Phi[W](x,b)-\Phi[\widetilde W](x,b)]}\\
&\le - \mu_h\|a-b\|^2 + L_\star\|a-b\|\|W-\widetilde W\|_\infty,
\end{align*}
so that 
\begin{equation}
\|S_h(W)-S_h(\widetilde W)\|_{\infty,2}
\le\frac{L_\star}{\mu_h}\|W-\widetilde W\|_\infty.
\label{eq:selector-lipschitz}
\end{equation}
Define an operator $\cT_h:\cB\to\cB$ by
$$
(\cT_hW)^i(x):=Q^i[W](x,S_h(W)(x)).
$$
The identity $hC_r+\beta^iM^i_\star=M^i_\star$ shows that $\cT_h$ indeed maps $\cB$ into itself. 
Next, let $a,b$ be defined as above. By the triangle inequality, 
the Lipschitz estimate \eqref{eq:Q-action-lip} and \eqref{eq:selector-lipschitz}, we obtain
\begin{align*}
|Q^i[W](x,a)-Q^i[\widetilde W](x,b)| &\le |Q^i[W](x,a)-Q^i[W](x,b)| + |Q^i[W](x,b)-Q^i[\widetilde W](x,b)|\\
&\le \Big(\frac{L_hL_\star}{\mu_h}  + \beta_\star \Big) \|W-\widetilde W\|_\infty
=\kappa_h\|W-\widetilde W\|_\infty.
\end{align*}
Taking the maximum over ${i\in[N]\text{ and }x\in\cX}$ shows
\begin{equation}
\|\cT_hW-\cT_h\widetilde W\|_\infty\le \kappa_h\|W-\widetilde W\|_\infty.
\label{eq:bellman-contraction}
\end{equation}
Banach's contraction theorem gives a unique fixed point $V^\star\in\cB$ since $\kappa_h<1$. 
Set $\alpha^\star=S_h(V^\star)$. By a standard verification argument, this yields the unique MPE.
\end{proof}

\section{Convergence proof}
\label{secLconvergence}
\begin{proof}[Proof of Theorem~\ref{thm:convergence}]
Let $(\alpha^\star,V^\star)$ be the unique MPE. 
In the proof below, $K<\infty$ denotes a deterministic constant, independent of $n$, 
whose value may change from line to line.

\vspace{-10pt}
\paragraph{Step 1: Uniform bounds.}
By assumption, $\|\widehat V^{i,0}\|_\infty\le{\overline M}$,
 and $M_\star= hC_r/(1-\beta_\star) \le \overline M$. 
 Suppose that at any step $n$, $\|\widehat V^{i,n}\|_\infty\le{\overline M}$. Then
$$
|\widetilde V^{i,n+1}|
\le hC_r+\beta^i{\overline M}
\le (1-\beta_\star)  {\overline M} + \beta_\star {\overline M}
={\overline M}.
$$

The value update is a convex combination of two numbers in
$[-{\overline M},{\overline M}]$, and unvisited entries do not change.
Induction therefore gives
\begin{equation}
|\widehat V^{i,n}(x)|\le{\overline M},
\qquad
|\widetilde V^{i,n+1}|\le{\overline M},
\qquad \forall\ i\in[N], x\in\cX, n\ge0.
\label{eq:iterate-bounds}
\end{equation}
Consequently, all policy and value estimates are also uniformly bounded.

\vspace{-10pt}
\paragraph{Step 2: Infinite visits.}
We claim that, for every $x\in\cX$,
\begin{equation}
\sum_{n=0}^{\infty}\eta_n\one_{\{X_n=x\}}=\infty,
\qquad
\sum_{n=0}^{\infty}\nu_n\one_{\{X_n=x\}}=\infty
\qquad\text{a.s.}
\label{eq:weighted-visitation}
\end{equation}
Indeed, fix $x\in\cX$. For $n\ge1$, conditioning first on $\cG_{n-1}$ and then on $\cF_{n-1}$,
$$
\P(X_n=x\mid\cF_{n-1})
=\E[\P(x\mid X_{n-1},A_{n-1})\mid\cF_{n-1}]
\ge\zeta.
$$
We first prove the claim for $\eta_n$. The numerical assumptions give $\sum_n\eta_n=\infty$ and $\sum_n\eta_n^2<\infty$, the latter follows from \eqref{eq:policy-sequences} since $\delta_n\le1$ eventually.

Consider the series
$$
\sum_{n=1}^{\infty}\eta_n\, 
(\one_{\{X_n=x\}}-\P(X_n=x\mid\cF_{n-1})).
$$
Each summand is $\cF_n$-measurable, has conditional mean zero given $\cF_{n-1}$, and a second moment of at most $\eta_n^2$.
Its partial sums therefore form a martingale which is bounded in $L^2$, and the martingale convergence theorem implies that the series converges almost surely.

Now write
$$
\begin{aligned}
\sum_{n=1}^{m}\eta_n\one_{\{X_n=x\}}
={}&\sum_{n=1}^{m}\eta_n\P(X_n=x\mid\cF_{n-1})+\sum_{n=1}^{m}\eta_n
\bigl(\one_{\{X_n=x\}}-\P(X_n=x\mid\cF_{n-1})\bigr).
\end{aligned}
$$
The first term is at least $\zeta\sum_{n=1}^{m}\eta_n$, which tends to infinity. The second converges to a finite limit a.s., proving the first assertion in \eqref{eq:weighted-visitation}.

The same argument with $\nu_n$ in place of $\eta_n$ proves the second
assertion, using the assumption \eqref{eq:value-sequences}. 

\vspace{-10pt}
\paragraph{Step 3: Conditional estimates.}
We use the definition of $\widetilde V^{i,n+1}$ in \eqref{eq:bellman-sample} 
and Assumption \ref{asm:observations} on the 
observations to arrive at
\begin{align}
\nonumber
\E[\widetilde V^{i,n+1}\mid\cG_n] &= 
\E[\fR_{n+1}^i+\beta^i\widehat V^{i,n}(X_{n+1}) \mid\cG_n]
= hr^i(X_n,A_n)+\beta^i \E[\widehat V^{i,n}(X_{n+1}) \mid\cG_n]\\
\label{eq:sample-conditional-mean}
&= hr^i(X_n,A_n)+\beta^i\mathsf P[\widehat V^{i,n}](X_n,A_n) =Q^i[\widehat V^n](X_n,A_n).
\end{align}
Consider Lemma~\ref{lem:gradient} with $\sF=\cF_n$, $x=X_n$, $a=\widehat\alpha^n(X_n)$, $W=\widehat V^n$, 
$\delta=\delta_n$, $\lambda=\lambda_n$, and $\widetilde V=\widetilde V^{n+1}$.
By \eqref{eq:iterate-bounds} $W, \widetilde V$ are bounded by $\overline  M$.
Hence, in view of \eqref{eq:sample-conditional-mean} the gradient
lemma applies and we conclude that
there is a constant $K$, independent of $n$, satisfying
$$
\|\E[\widehat D^{n+1}\mid\cF_n]
-\Phi[\widehat V^n](X_n,\widehat\alpha^n(X_n))\|
\le K(\delta_n+\lambda_n).
$$
Adding and subtracting $\Phi[V^\star](X_n,\widehat\alpha^n(X_n))$ 
and using the Lipschitz estimate \eqref{eq:Phi-W-lip}, we obtain
\begin{equation}
\|\E[\widehat D^{n+1}\mid\cF_n]
-\Phi[V^\star](X_n,\widehat\alpha^n(X_n))\|\le
L_\star\|\widehat V^n-V^\star\|_\infty+K(\delta_n+\lambda_n).
\label{eq:gradient-mean}
\end{equation}
Additionally, we also have the error estimate,
\begin{equation}
\E\bigl[
\|\widehat D^{n+1}-\E[\widehat D^{n+1}\mid\cF_n]\|^2
\mid\cF_n\bigr]
\le K\delta_n^{-2}.
\label{eq:gradient-noise}
\end{equation}

We next use \eqref{eq:sample-conditional-mean},
the Nash equation \eqref{eq:equilibrium-value}, and 
the definition of the Bellman operator $Q^i$ to arrive at
\begin{align*}
\bigl|\E[\widetilde V^{i,n+1}\mid\cG_n] -V^{i,\star}(X_n)\bigr|
&= \bigl|Q^i[\widehat V^n](X_n,A_n)-Q^i[V^\star](X_n,\alpha^\star(X_n))\bigr|\\
&  = \bigl|\beta^i\mathsf P[\widehat V^{i,n}-V^{i,\star}](X_n,A_n)+
Q^i[V^\star](X_n,A_n) -Q^i[V^\star](X_n,\alpha^\star(X_n))\bigr|\\
& \le \beta^i\|\widehat V^n-V^\star\|_\infty + 
L_h\|A_n-\alpha^\star(X_n)\|,
\end{align*}
where we use the Lipschitz estimate
\eqref{eq:Q-action-lip} at $V^\star$ and
the identity $L_Q(M_\star)=L_h$.
By the definition of $A_n$, 
\begin{equation}
\|A_n-\widehat\alpha^n(X_n)\| =\|-\lambda_n \widehat\alpha^n(X_n)
+ \delta_n \Delta_n \|
\le\lambda_n\max_{u\in\cU}\|u\|+\sqrt N\,\delta_n.
\label{eq:perturbation-distance}
\end{equation}
The right-hand side can be estimated by $K (\delta_n +\lambda_n)$
for some  constant  $K$, independent of $n$.
Combining these inequalities, conditioning on $\cF_n\subset\cG_n$, and using the tower property, we arrive at
\begin{equation}
\label{eq:value-target-bound}
\bigl|\E[\widetilde V^{i,n+1}\mid\cF_n] -V^{i,\star}(X_n)\bigr| \le
\beta^i\|\widehat V^n-V^\star\|_\infty
+ L_h\|\widehat\alpha^n(X_n)-\alpha^\star(X_n)\|
+ K (\delta_n +\lambda_n).
\end{equation}

\vspace{-10pt}
\paragraph{Step 4: Policy error bound.}
By Step~1, the following quantities are finite,
$$
\cE_V:=\underset{\Omega}{\mathrm{ess\,sup}}\,
\big(\limsup_{n\to\infty}\|\widehat V^n-V^\star\|_\infty\big),
\qquad
\cE_\alpha:=\underset{\Omega}{\mathrm{ess\,sup}}\, 
\big(\limsup_{n\to\infty}\|\widehat\alpha^n-\alpha^\star\|_{\infty,2}\big).
$$
Fix $x\in\cX$, and write
$\alpha_n:=\widehat\alpha^n(x)$, 
$\alpha^\star:= \alpha^\star(x)$, and
$\Psi(a):=-\Phi[V^\star](x,a)$.
The function $\Psi$ is bounded on $\cU$. By
\eqref{eq:equilibrium-VI} and \eqref{eq:bellman-monotonicity},
\begin{equation}
\ip{\alpha-\alpha^\star}{\Psi(\alpha)}
\ge\mu_h\|\alpha-\alpha^\star\|^2,
\qquad \alpha \in \cU.
\label{eq:policy-coercivity}
\end{equation}
Further set
$k_n:= \eta_n\one_{\{X_n=x\}}$, $b_n:=\one_{\{X_n=x\}}
\bigl(-\E[\widehat D^{n+1}\mid\cF_n]-\Psi(\alpha_n)\bigr)$,
$\xi_{n+1} :=\E[\widehat D^{n+1}\mid\cF_n]-\widehat D^{n+1}$,
so that the policy update equation \eqref{eq:value-update}
can be written as 
$$
\alpha_{n+1}
=\Proj_{\cU}\left(
\alpha_n-k_n\bigl(\Psi(\alpha_n)+b_n+\xi_{n+1}\bigr) \right).
$$
When $X_n=x$, the expression in parentheses equals
$-\widehat D^{n+1}$, so the update coincides with \eqref{eq:value-update}.
Otherwise, the entry remains unchanged.

The random variable $b_n$ is $\cF_n$-measurable, and
\eqref{eq:gradient-mean}, together with $\delta_n,\lambda_n\to0$,
gives
$$
\limsup_{n\to\infty}\|b_n\|\le L_\star \cE_V
\qquad\text{a.s.}
$$
Towards applying Proposition~\ref{prop:sa}, observe that 
$k_n=\eta_n\one_{\{X_n=x\}} \ge 0$ is
$\cF_n$-measurable, and by  \eqref{eq:weighted-visitation} the series is divergent.
Also, $\xi_{n}\in L^2(\cF_{n})$, $\E[\xi_{n+1}\mid\cF_n]=0$, and 
by \eqref{eq:gradient-noise}
$$
\sum_n  k_n^2
\left(1+\E[\|\xi_{n+1}\|^2\mid\cF_n]\right)\le
\sum_n\eta_n^2+K\sum_n\frac{\eta_n^2}{\delta_n^2}<\infty.
$$
In view of  \eqref{eq:policy-coercivity} and the above observations, 
Proposition~\ref{prop:sa} applies 
with the set $\cU$, target $\alpha^*=\alpha^\star(x)$, yielding
$$
\limsup_{n\to\infty}\|\widehat\alpha^n(x)-\alpha^\star(x)\|
=\limsup_{n\to\infty}\|\alpha_n-\alpha^\star\|
\le L_\star \cE_V /\mu_h \qquad\text{a.s.},
$$
for every $x \in \cX$. Taking the maximum over the finitely many states,
we conclude that
\begin{equation}
\cE_\alpha \le L_\star \cE_V /\mu_h \qquad\text{a.s.}
\label{eq:policy-gain}
\end{equation}

\vspace{-10pt}
\paragraph{Step 5: Value error bound.}
Fix $i\in[N]$, $x\in\cX$. Write
$v_n:=\widehat V^{i,n}(x)$, $v^\star:=V^{i,\star}(x)$, and set
$$
\widehat b_n:=\one_{\{X_n=x\}}
\bigl(v^\star-\E[\widetilde V^{i,n+1}\mid\cF_n]\bigr),\qquad 
\widehat \xi_{n+1} :=\E[\widetilde V^{i,n+1}\mid\cF_n]-\widetilde V^{i,n+1}.
$$
We rewrite the value update \eqref{eq:value-update}
in the form utilized in Proposition~\ref{prop:sa} as follows
$$
v_{n+1} =\Proj_{[-{\overline M},{\overline M}]}
\left(v_n-\widehat k_n\bigl(\widehat\Psi(v_n)+\widehat b_n+ \widehat \xi_{n+1}\bigr)\right),
$$
where $\widehat k_n:=\nu_n\one_{\{X_n=x\}}$ and $\widehat \Psi(v):= (v-v^\star)$.
As in the previous step, when $X_n=x$, the expression in parentheses
equals $v_n-\widetilde V^{i,n+1}$, so the update coincides with
\eqref{eq:value-update}. Otherwise, the entry remains unchanged.

Both $v_n$ and $v^\star$ lie in $[-{\overline M},{\overline M}]$.
The map $v\mapsto \widehat \Psi(v)$ satisfies \eqref{eq:sa-coercivity}
with $\mu=1$.
The random variables $\widehat b_n$ are $\cF_n$-measurable, and
\eqref{eq:value-target-bound} and \eqref{eq:policy-gain} give
$$
\limsup_{n\to\infty}|\widehat b_n| \le \beta_\star \cE_V+L_h\cE_\alpha
\le\beta_\star \cE_V+L_h\frac{L_\star}{\mu_h}\, \cE_V
=\kappa_h\, \cE_V
\qquad\text{a.s.}
$$
Moreover, $\widehat\xi_{n+1} $ is $\cF_{n+1}$-measurable and satisfies 
$\E[\widehat\xi_{n+1} \mid\cF_n]=0$. 
Using $|\widehat\xi_{n+1} |\le2{\overline M}$,  
$$
\sum_n\bigl(\nu_n\one_{\{X_n=x\}}\bigr)^2
\left(1+\E[|\widehat\xi_{n+1} |^2\mid\cF_n]\right)\le
\bigl(1+4\overline M^2\bigr)\sum_n\nu_n^2<\infty.
$$
Finally, $\widehat k_n=\nu_n\one_{\{X_n=x\}} \ge 0$ is
$\cF_n$-measurable and has an infinite sum by
\eqref{eq:weighted-visitation}.

The above observations imply that we may apply Proposition~\ref{prop:sa} to obtain
$$
\limsup_{n\to\infty}|\widehat V^{i,n}(x)-V^{i,\star}(x)|
\le\kappa_h\, \cE_V \qquad\text{a.s.},
$$
for every $x \in \cX$.  Taking the maximum over the finitely many agents 
and states shows that $\cE_V \le \kappa_h \, \cE_V$.
Since $\kappa_h<1$ by hypothesis and $\cE_V\ge 0$ by definition,
this implies that $\cE_V=0$.

\vspace{-10pt}
\paragraph{Step 6: Conclusion.}
We have shown that $\cE_V=0$ and, by \eqref{eq:policy-gain},
we also have $\cE_\alpha=0$.
Finally, by \eqref{eq:perturbation-distance},
\begin{equation*}
\|A_n-\alpha^\star(X_n)\|
\le\|\widehat\alpha^n-\alpha^\star\|_{\infty,2}
+\lambda_n\max_{u\in\cU}\|u\|+\sqrt N\,\delta_n
\longrightarrow0
\qquad\text{a.s.}\qedhere
\end{equation*}
\end{proof}

\section{Proof of complexity}
\label{sec:rate}

We use the following elementary fact.

\begin{Lem}
\label{lem:polynomial-comparison}
Let $(u_n)_{n\ge1}$ be a sequence of non-negative real numbers. 
Suppose that, for some constants $\gamma \ge 0$ and $\kappa>\theta>0$,
$$
u_{n+1}\le\left(1-\frac{\kappa}{n}\right)u_n
+\frac{\gamma}{n^{1+\theta}}
$$
for all sufficiently large $n$. Then there is a constant $\widehat c$
such that $u_n\le \widehat c \, n^{-\theta}$ for every $n\ge1$.
\end{Lem}

\begin{proof}
Choose an integer $n_0>\kappa$ after which the recursion holds, and
choose $\widehat c$ so  that $u_n\le \widehat c\, n^{-\theta}$ 
for $1\le n\le n_0$ and $\gamma \le \widehat c\, (\kappa-\theta)$. 
If $u_n\le \widehat c\, n^{-\theta}$ for $n\ge n_0$, then
$$
u_{n+1}\le \left(1-\frac{\kappa}{n}\right)\, \widehat c\, n^{-\theta} 
+ \frac{\widehat c\, (\kappa-\theta)}{n^{1+\theta}}
=\widehat c\, n^{-\theta}\left(1-\frac{\theta}{n}\right)
\le \widehat c\, (n+1)^{-\theta},
$$
where the last inequality follows from
$(1+t)^{-\theta}\ge1-\theta t$. The claim follows by induction.
\end{proof}

\begin{proof}[Proof of Theorem~\ref{thm:rate}]
Throughout the proof, $K<\infty$ denotes a deterministic constant,
independent of $n$, that may change from line to line.

\vspace{-10pt}
\paragraph{Step 1: Error measurement.} 
In this step, we rewrite the error 
$\mathsf A^n+\mathsf V^n$ in a suitable form.
To balance the policy and value errors, we choose a constant
$s_\star>L_\star/\mu_h$ as follows. Set $f(s):=\max\left\{L_\star/(\mu_hs)\ ;\ \beta_\star+L_hs\right\}$. 
In view of the contraction condition \eqref{eq:contraction}
$\kappa_h=\beta_\star+L_h (L_\star/\mu_h)<1$,
and therefore there is $s_\star >0$ satisfying $f(s_\star)<1$.
Set
$$
q_\star:=f(s_\star)= \max\left\{L_\star/(\mu_hs_\star)\ ;\ \beta_\star+L_hs_\star\right\}<1.
$$

Let $\zeta>0$ be the ergodicity constant in Assumption~\ref{asm:accessibility}. 
Choose an even integer $p\ge4$ satisfying
$$
\frac{q_\star(N+1)}{p}\le\frac{\zeta(1-q_\star)}4,
\quad \text{set}\quad 
\omega_\star:=\frac1{\zeta(1-q_\star)},\quad \text{and fix}\quad \omega>\omega_\star.
$$
Consider the point-wise errors at step $n$ given by
$$
e_n(x):=\frac{1}{s_\star} [\widehat\alpha^n(x)-\alpha^\star(x)],
\qquad
v_n^i(x):=\widehat V^{i,n}(x)-V^{i,\star}(x),\qquad x \in \cX,
$$
and define point-wise and total errors by
$$
s_n(x):=\|e_n(x)\|^p+\sum_{i=1}^N|v_n^i(x)|^p,
\qquad S_n:=\sum_{x\in\cX}s_n(x),\qquad x \in \cX.
$$
Then, $\|e_n(x)\|, |v_n^i(x)| \le S_n^{1/p}$ for each $n,i,x$.
Since $\E[|Z|^2]\le\E[|Z|^p]^{2/p}$ by Jensen's inequality, we have
\begin{equation}
\mathsf A^n+\mathsf V^n
\le |\cX| (s_\star^2+N)\, \E[S_n]^{2/p}.
\label{eq:rate-error-comparison}
\end{equation}
Therefore, it suffices to prove $\E[S_n]\le Kn^{-p/4}$. 
The remaining proof establishes this inequality.

\vspace{-10pt}
\paragraph{Step 2: Iteration error.}
We claim that, for every $\varepsilon>0$, there is $K_\varepsilon<\infty$
such that
\begin{equation}
\label{eq:rate-one-step}
\E[S_{n+1}\mid\cF_n] \le \left(1+\frac{\omega q_\star(N+1)
+\varepsilon}{n}\right)S_n -\frac{p\omega(1-q_\star)}n s_n(X_n)
+K_\varepsilon n^{-1-p/4}.
\end{equation}
Indeed, at step $n$ the updates change only at $X_n$. Hence,
$S_{n+1} -S_n = s_{n+1}(X_n)-s_n(X_n)$.
Put $\fz=X_n$ and observe that 
$$
s_{n+1}(\fz)=s_n(\fz) +\|e_{n+1}(\fz)\|^p -\|e_n(\fz)\|^p
+ \sum_{i=1}^N\big(|v_{n+1}^i(\fz)|^p-|v_n^i(\fz)|^p\big).
$$
We estimate the policy and value increments separately.
As the projections are non-expansive, the definitions of the updates imply
$$
\|e_{n+1}(\fz)\| \le\big\|e_n(\fz)+\frac{\eta_n}{s_\star} \widehat D^{n+1}\big\|,\qquad
|v_{n+1}^i(\fz)| =\big|v_n^i(\fz)+\nu_n
(\widetilde V^{i,n+1}-\widehat V^{i,n}(\fz))\big|.
$$
We bound these terms using the estimate for some constant $K$,
\begin{equation}
\label{eq:rate-proof-taylor}
\|y+t\|^p
\le\|y\|^p+p\|y\|^{p-2}\ip{y}{t}
+K\bigl(\|y\|^{p-2}\|t\|^2+\|t\|^p\bigr),
\end{equation}
which follows from 
Taylor's formula applied to $\|\cdot\|^p$, for any
vectors $y,t$ of the same dimension. To apply this, we next estimate the cross-terms that appear.

Firstly, \eqref{eq:policy-coercivity} gives
$$
\frac1{\mu_hs_\star}
\ip{e_n(\fz)}{\Phi[V^\star](\fz,\widehat\alpha^n(\fz))}
\le-\|e_n(\fz)\|^2.
$$
Adding and subtracting the pseudo-gradient and using the equation 
\eqref{eq:equilibrium-value}, the coercivity \eqref{eq:policy-coercivity}, 
Cauchy--Schwarz and \eqref{eq:gradient-mean}, yield
\begin{equation}
\begin{aligned}
\frac1{\mu_h s_\star}\ip{e_n(\fz)}{\E[\widehat D^{n+1}\mid\cF_n]}
&\le-\|e_n(\fz)\|^2
+\frac{\|e_n(\fz)\|}{\mu_h s_\star}
\bigl\|\E[\widehat D^{n+1}\mid\cF_n]
-\Phi[V^\star](\fz,\widehat\alpha^n(\fz))\bigr\|\\
&\le-\|e_n(\fz)\|^2
+\frac{L_\star}{\mu_h s_\star}\|e_n(\fz)\|
\|\widehat V^n-V^\star\|_\infty +K(\delta_n+\lambda_n)\|e_n(\fz)\|\\
&\le-\|e_n(\fz)\|^2 +q_\star \|e_n(\fz)\|S_n^{1/p}
+K(\delta_n+\lambda_n)\|e_n(\fz)\|.
\end{aligned}\label{eq:rate-proof-bound-action}
\end{equation}
The last inequality uses $L_\star/(\mu_h s_\star)\le q_\star$ and
$\|\widehat V^n-V^\star\|_\infty\le S_n^{1/p}$. 

Similarly,
\eqref{eq:value-target-bound} gives
\begin{align}
\nonumber
v_n^i(\fz)\bigl(\E[\widetilde V^{i,n+1}\mid\cF_n] -\widehat V^{i,n}(\fz)\bigr)
&=-|v_n^i(\fz)|^2 +v_n^i(\fz)\bigl(\E[\widetilde V^{i,n+1}\mid\cF_n] -V^{i,\star}(\fz)\bigr)\\
\nonumber
&\le-|v_n^i(\fz)|^2 +|v_n^i(\fz)|\big( \beta_\star\|\widehat V^n-V^\star\|_\infty
+L_hs_\star\ \|e_n(\fz)\|+K(\delta_n+\lambda_n)\big)\\
&\le-|v_n^i(\fz)|^2 +q_\star |v_n^i(\fz)|S_n^{1/p} +K(\delta_n+\lambda_n)|v_n^i(\fz)|.
\label{eq:rate-proof-bound-value}
\end{align}
The last inequality uses
$\|\widehat V^n-V^\star\|_\infty\le S_n^{1/p}$,
$\|e_n(\fz)\|\le S_n^{1/p}$, and $\beta_\star+L_hs_\star \le q_\star$.

We apply the Taylor estimate \eqref{eq:rate-proof-taylor} to obtain
\begin{align*}
\E[S_{n+1}\mid\cF_n] \le S_n
&+\frac{p\omega}{n}\|e_n(\fz)\|^{p-2} \frac1{\mu_hs_\star}
\ip{e_n(\fz)}{\E[\widehat D^{n+1}\mid\cF_n]}\\
&+\frac{p\omega}{n}\sum_i|v_n^i(\fz)|^{p-2}v_n^i(\fz)
\bigl(\E[\widetilde V^{i,n+1}\mid\cF_n] -\widehat V^{i,n}(\fz)\bigr)\\
&+\frac{K}{n^2\delta_n^2} \Big(\|e_n(\fz)\|^{p-2}+\sum_i|v_n^i(\fz)|^{p-2}\Big)
+\frac{K}{n^p\delta_n^p}.
\end{align*}
Here we used $\eta_n=\omega/(\mu_hn)$, $\nu_n=\omega/n$,
and $\delta_n\le1$ eventually. We next plug in the cross-term estimates \eqref{eq:rate-proof-bound-action} and
\eqref{eq:rate-proof-bound-value} in the first two lines.
The result is
$$
\begin{aligned}
&\E[S_{n+1}\mid\cF_n]
\le S_n\,\underbrace{-\,\frac{p\omega}{n}s_n(\fz)+\frac{p\omega q_\star}{n}S_n^{1/p}
\Big(\|e_n(\fz)\|^{p-1}+\sum_i|v_n^i(\fz)|^{p-1}\Big)}_{\mathrm{(A)}}\\
&\qquad +\underbrace{\frac{K(\delta_n+\lambda_n)}{n}
\Big(\|e_n(\fz)\|^{p-1} + \sum_i|v_n^i(\fz)|^{p-1}\Big)}_{\mathrm{(B)}}+
\underbrace{\frac{K}{n^2\delta_n^2}
\Big(\|e_n(\fz)\|^{p-2}+\sum_i|v_n^i(\fz)|^{p-2}\Big)}_{\mathrm{(C)}}
+\underbrace{\frac{K}{n^p\delta_n^p}}_{\mathrm{(D)}}.
\end{aligned}
$$
We first consider the terms in (A). By Young's inequality, for any $r>0$,
$$
r^{p-1}S_n^{1/p}
\le\frac{p-1}{p}r^p+\frac1pS_n.
$$
Applying this to $r=\|e_n(\fz)\|,|v_n^1(\fz)|,\ldots,|v_n^N(\fz)|$
and adding these inequalities shows that
\begin{align*}
\mathrm{(A)} &\le
-\frac{\omega[p-q_\star(p-1)]}{n}s_n(\fz)
+\frac{\omega q_\star(N+1)}{n}S_n\le
-\frac{p\omega(1-q_\star)}{n}s_n(\fz)
+\frac{\omega q_\star(N+1)}{n}S_n.
\end{align*}
The last inequality uses $p-q_\star(p-1)=p(1-q_\star)+q_\star$.

We continue estimating (B)--(D). Since every error is at most $S_n^{1/p}$, 
the chosen learning schedules imply
$$
\begin{aligned}
\mathrm{(B)} \le Kn^{-5/4}S_n^{1-1/p},\qquad 
\mathrm{(C)}
\le Kn^{-3/2}S_n^{1-2/p},\qquad \mathrm{(D)}\le Kn^{-3p/4},
\end{aligned}
$$
after increasing $K$ if necessary.
For any $\varepsilon>0$, Young's inequality further gives $K_\varepsilon<\infty$ with
$$
\begin{aligned}
\mathrm{(B)} \le Kn^{-5/4}S_n^{1-1/p}
\le\frac{\varepsilon}{2n}S_n
+K_\varepsilon n^{-1-p/4},\qquad 
\mathrm{(C)} \le Kn^{-3/2}S_n^{1-2/p}
\le\frac{\varepsilon}{2n}S_n
+K_\varepsilon n^{-1-p/4}.
\end{aligned}
$$
Finally, since $p\ge4$, we have $3p/4\ge1+p/4$, and hence
$$
\mathrm{(D)}\le Kn^{-3p/4}\le Kn^{-1-p/4}.
$$
Thus, (D) can be absorbed into the remainder $K_\varepsilon n^{-1-p/4}$ by 
increasing $K_\varepsilon$. Combining the bounds yields
\eqref{eq:rate-one-step} for sufficiently large $n$.

\vspace{-10pt}
\paragraph{Step 3: Total error.}
We next claim that, for each $n$,  the following reverse inequality holds:
\begin{equation}
\E[s_n(X_n)]\ge\frac{\zeta}{2}\ \E[S_n]-Kn^{-3p/4}.
\label{eq:rate-visited-error}
\end{equation}
We first consider $e_n$.  Since $\widehat\alpha^{n-1}(x) \in \cU$,
its projection is equal to itself and therefore
$$
s_\star(e_n(x)-e_{n-1}(x)) = \Proj_{\cU}(\widehat\alpha^{n-1}(x)
+\one_{\{X_{n-1}=x\}}\eta_{n-1}\widehat D^n) - \Proj_{\cU}(\widehat\alpha^{n-1}(x)).
$$
This implies that
$$
\|e_n(x)-e_{n-1}(x)\| \le  \frac{\eta_{n-1}}{s_\star} \|\widehat D^n\|
\le K \frac{\eta_{n-1}}{\delta_{n-1}},
$$
for some constant $K$.
For the values, the update formula gives
$$
|v_n^i(x)-v_{n-1}^i(x)| =\nu_{n-1} \one_{\{X_{n-1}=x\}}
\bigl|\widetilde V^{i,n}-\widehat V^{i,n-1}(x)\bigr|
\le2\overline M\,\nu_{n-1}.
$$
Using the learning rates \eqref{eq:rate-schedules}, we obtain
\begin{align*}
&\max_x\big(\|e_n(x)-e_{n-1}(x)\|
+\max_i|v_n^i(x)-v_{n-1}^i(x)|\big)\\
&\qquad \qquad \le
K(\frac{\eta_{n-1}}{\delta_{n-1}}+\nu_{n-1})\le
K\big((n-1)^{-3/4}+(n-1)^{-1}\big)
\le Kn^{-3/4}.
\end{align*}
Using Young's inequality $\|y+t\|^p\le\tfrac43\|y\|^p+K\|t\|^p$, we obtain
$$
\begin{aligned}
\|e_{n}(x)\|^p
&=\|e_{n-1}(x)+(e_{n}(x)-e_{n-1}(x))\|^p
\le \tfrac43 \|e_{n-1} (x)\|^p+Kn^{-3p/4}.
\end{aligned}
$$
The same argument gives $|v_n^i(x)|^p \le \tfrac43 |v_{n-1}^i(x)|^p+Kn^{-3p/4}$.
Adding over the policy and value errors, and increasing $K$ when needed, yields 
$s_{n}(x)\le \tfrac43 s_{n-1}(x) + Kn^{-3p/4}$.
Simple algebra  implies that
$s_{n-1}(x)\ge \tfrac34 s_n(x) - Kn^{-3p/4}$  for every $\ x\in \cX$.
Summing this inequality over $x\in\cX$ and again  increasing $K$ when needed, gives
$S_{n-1}\ge\tfrac34 S_n  -  Kn^{-3p/4}$.

Now we rearrange the terms differently and then use Young's inequality to obtain
$$
\|e_{n-1}(x)\|^p =\|e_n(x)+(e_{n-1}(x)-e_n(x))\|^p
\le \tfrac43 \|e_n (x)\|^p+Kn^{-3p/4}.
$$
This argument also gives $|v_{n-1}^i(x)|^p \le \tfrac43 |v_n^i(x)|^p+Kn^{-3p/4}$,
and arguing exactly as before
we obtain $s_n(x)\ge \tfrac34 s_{n-1}(x)-  Kn^{-3p/4}$  for every 
$\ x\in \cX$.  Thus, we have shown the following inequalities:
$$
S_{n-1}\ge\tfrac34 S_n - Kn^{-3p/4}, \qquad
s_{n}(x)\ge \tfrac34 s_{n-1}(x) - Kn^{-3p/4}, \ \ \forall\ x\in \cX.
$$

In Step~2 of the proof of Theorem~\ref{thm:convergence} we have
directly shown that  $\P(X_n=x\mid\cF_{n-1})\ge\zeta$.
Since $s_{n-1}(x)$ is $\cF_{n-1}$-measurable,
we directly estimate that
$$
\begin{aligned}
\E[s_n(X_n)]&\ge \frac34 \E[s_{n-1}(X_n)]-Kn^{-3p/4}
\ge\frac34\E\Big[\sum_x \P(X_n=x\mid\cF_{n-1})s_{n-1}(x)\Big]-Kn^{-3p/4}\\
& \ge \frac{3\zeta}{4}\E[S_{n-1}]-Kn^{-3p/4}
\ge\frac{9\zeta}{16}\ \E[S_n]-Kn^{-3p/4}
\ge\frac{\zeta}{2}\ \E[S_n]-Kn^{-3p/4}.
\end{aligned}
$$

\vspace{-10pt}
\paragraph{Step 4: Conclusion.}
Set $U_n:=\E[S_n]$. Taking expectations in
\eqref{eq:rate-one-step} and applying
\eqref{eq:rate-visited-error}, we obtain, for all sufficiently large $n$,
\begin{align*}
U_{n+1}
&\le \left(1+\frac{\omega q_\star(N+1)+\varepsilon}{n}\right)U_n
-\frac{p\omega(1-q_\star)}{n}
\left(\frac{\zeta}{2}U_n-Kn^{-3p/4}\right)
+K_\varepsilon n^{-1-p/4}\\
&= \left(1-\frac{p\omega\zeta(1-q_\star)/2-\omega q_\star(N+1)-\varepsilon}{n}\right)U_n
+K_\varepsilon n^{-1-p/4} +p\omega(1-q_\star)K n^{-1-3p/4}.
\end{align*}
Since $p>0$ and $n\ge1$, we have
$n^{-1-3p/4}\le n^{-1-p/4}$, so the last term can be
absorbed into the second last by increasing $K_\varepsilon$. 
Moreover, our choice of $p$ gives
$$
\omega q_\star(N+1)\le\frac{p\omega\zeta(1-q_\star)}4,
\qquad \Longrightarrow\qquad 
\frac{p\omega\zeta(1-q_\star)}2-\omega q_\star(N+1)-\varepsilon
\ge\frac{p\omega\zeta(1-q_\star)}4-\varepsilon.
$$
Since $U_n\ge0$, we therefore can conclude that
$$
U_{n+1} \le\left(1-\frac{p\omega\zeta(1-q_\star)/4-\varepsilon}{n}\right)U_n
+K_\varepsilon n^{-1-p/4}.
$$
Because $\omega>\omega_\star=1/(\zeta(1-q_\star))$, we can choose
$\varepsilon>0$ so that $p\omega\zeta(1-q_\star)-4 \varepsilon>p$.
We apply Lemma~\ref{lem:polynomial-comparison} with $\theta=p/4$ to
conclude that $U_n\le Kn^{-p/4}$. Together with \eqref{eq:rate-error-comparison}, this establishes the claim \eqref{eq:mean-square-rate}.

Each squared maximum in \eqref{eq:max-rate} is bounded by the
corresponding sum in $\mathsf A^n+\mathsf V^n$. Finally,
\eqref{eq:perturbation-distance} gives
\begin{equation*}
\E\bigl[\|A_n-\alpha^\star(X_n)\|^2\bigr]
\le2\mathsf A^n+K(\delta_n+\lambda_n)^2
\le\frac K{\sqrt n}.
\qedhere
\end{equation*}
\end{proof}

\section{Episodic learning}
\label{sec:episodic}

This section discusses the convergence 
guarantee for the episodic algorithm given in Subsection \ref{ss:episodic}.

\vspace{-10pt}
\paragraph{Assumptions.} We retain the analogous assumptions of the discounted problem. 
Specifically, we assume that  the time-dependent versions of 
Assumptions~\ref{asm:rewards} and~\ref{asm:transition} hold uniformly in $t$, and $|g^i(x)|\le C_g$. Put
{\begin{equation*}
M_T:=C_g+ThC_r,
\qquad
\mu_T:=h(\mu- hN M_Tc_2).
\end{equation*}}
As in Assumption~\ref{asm:gap}, we assume that
$\mu_T>0$, but the contraction condition $\kappa_h<1$ is \emph{not} required.

With deterministic initial tables, let $\cE_0:=\{\varnothing,\Omega\}$ and 
for $n\ge0$, $t<T$, recursively define
$$
\cF_0^n:=\cE_n\vee\sigma(X_0^n),\quad
\cG_t^n :=\cF_t^n\vee\sigma(\Delta_t^n),\quad
\cF_{t+1}^n :=\cG_t^n\vee\sigma(\fR_{t+1}^n,X_{t+1}^n),\quad
\cE_{n+1}:=\cF_T^n\vee\sigma(\fg^n).
$$
Thus $\cE_n$ contains the completed history before episode $n$, including the current tables. 
The updated tables are $\cE_{n+1}$-measurable. Conditional on $\mathcal F_t^n$, the perturbations $(\Delta_t^{i,n})_{i=1}^N$ are independent across agents and satisfy the moment conditions in $\eqref{eq:direction-moments}$. Almost surely, for every $n\ge0$, $t<T$, $i\in[N]$, and $y\in\cX$, the observations satisfy
$$
\P(X_{t+1}^n=y\mid\cG_t^n)=\P_t(y\mid X_t^n,A_t^n),\qquad
\E[\fR_{t+1}^{i,n}\mid\cG_t^n]=hr_t^i(X_t^n,A_t^n),
\qquad |\fR_{t+1}^{i,n}|\le hC_r,
$$
where $A_t^n=(A_t^{i,n})_{i=1}^N$ is the joint action. 
For the terminal reward, we assume that
$$
\E[\fg^{i,n}\mid\cF_T^n]=g^i(X_T^n),
\qquad  |\fg^{i,n}|\le C_g,
$$
for some constant $C_g$. Assume also that, for a deterministic $\zeta_T>0$,
$$
\P(X_t^n=x\mid\cE_n)\ge\zeta_T, \qquad \forall\, t\le T, x\in\cX, n\ge0.
$$
The numerical sequences satisfy \eqref{eq:basic-sequences}--\eqref{eq:value-sequences}, 
now indexed by episodes. We initialize deterministic feasible policies and value tables satisfying
$$
|\widehat V_t^{i,0}(x)|\le C_g+(T-t)hC_r.
$$

With these assumptions, we obtain the following result.
\begin{Thm}[Episodic equilibrium and learning]
\label{thm:episodic}
Under the assumptions stated above, there is a unique finite-horizon 
MPE $(\alpha^\star,V^\star)$, characterized by \eqref{eq:episodic-nash-system}. Its values satisfy
$$
|V_t^{i,\star}(x)|\le C_g+(T-t)hC_r.
$$
The episodic algorithm converges almost surely:
\begin{equation}
\max_{t<T}\|\widehat\alpha_t^n-\alpha_t^\star\|_{\infty,2}
+\max_{t\le T}\|\widehat V_t^n-V_t^\star\|_\infty
\longrightarrow0.
\label{eq:episodic-convergence}
\end{equation}
\end{Thm}

The proof of Theorem~\ref{thm:episodic} follows by backward induction, 
using the estimates in the proof of Theorem~\ref{thm:convergence} and 
Proposition~\ref{prop:sa}, and is omitted.

\appendix

\section{Proof of the estimates}
\label{app:bellman}

\begin{proof}[Lemma~\ref{lem:bellman}]
Let $e:=a-b$ and fix $i\in[N]$. For $\|v^i\|\le1$, by the fundamental theorem of calculus,
\begin{align*}
\ip{\nabla_{a^i}\mathsf P [W^i](x,a)-\nabla_{a^i}\mathsf P [W^i](x,b)}{v^i}=\sum_{j=1}^N\int_0^1\sum_{y\in\cX}W^i(y)
D_{ji}^2\P(y\mid x,b+\theta e)[e^j,v^i]\,\d\theta.
\end{align*}
By \eqref{eq:transition}, the absolute value is at most
$h^2c_2\, \fm \sum_j\|e^j\|\,\|v^i\|$. Taking the supremum over $\|v^i\|\le1$ gives
$$
\|\nabla_{a^i}\mathsf P [W^i](x,a)-\nabla_{a^i}\mathsf P [W^i](x,b)\|
\le  h^2c_2 \fm \sum_{j=1}^N\|e^j\|.
$$
Multiplying each agent's component by $\beta^i$ and taking the
Euclidean norm across agents gives
$$
\Big(\sum_{i=1}^N(\beta^i)^2
\|\nabla_{a^i}\mathsf P [W^i](x,a)-\nabla_{a^i}\mathsf P [W^i](x,b)\|^2\Big)^{1/2}\le\sqrt N\,\beta_\star 
h^2c_2 \fm \sum_{j=1}^N\|e^j\|
\le\beta_\star N h^2c_2 \fm \|e\|,
$$
where we used $\sum_j\|e^j\|\le\sqrt N\|e\|$.
Hence, the continuation contribution to \eqref{eq:Phi-action-lip} is at most 
$\beta_\star N h^2c_2 \fm \|e\|$. 
The running-reward contribution is at most $hL_F\|e\|$. 
This establishes \eqref{eq:Phi-action-lip}.

Similarly, \eqref{eq:transition} gives
$$
\|\beta^i\nabla_{a^i}\mathsf P[W^i-\widetilde W^i](x,a)\| \le\beta^i hc_1\|W-\widetilde W\|_\infty.
$$
Taking the Euclidean norm across agents proves \eqref{eq:Phi-W-lip}.
A line-segment argument as before gives
$$
|\beta^i\mathsf P [W^i](x,a)-\beta^i\mathsf P [W^i](x,b)| \le\beta_\star  hc_1\fm \sum_{j=1}^N\|e^j\|
\le\beta_\star\sqrt N  hc_1\fm \|e\|,
$$
which, together with \eqref{eq:reward-bounds}, proves \eqref{eq:Q-action-lip}.

Finally, when $\|W\|_\infty\le M_\star$, the absolute value of the continuation 
contribution to the inner product in \eqref{eq:bellman-monotonicity} is at most
$$
\beta_\star M_\star h^2c_2\sum_{i,j}\|e^i\|\|e^j\|
\le\beta_\star NM_\star h^2c_2\|e\|^2.
$$
Combining this with \eqref{eq:running-monotonicity} proves the assertion.
\end{proof}

\section{Proof of the gradient-estimate}
\label{app:gradient}

\begin{proof}[Proof of Lemma~\ref{lem:gradient}]
Fix $i\in[N]$ and set $\fm:=\|W\|_\infty$, $\widetilde a^i=(1-\lambda)a^i$,
$A^i=\widetilde a^i + \delta \Delta^i$,
where $a \in \cU$ is the policy fixed in the statement.
Since $\E[\widetilde V^i\mid\sG]=Q^i[W](x,A)$ by hypothesis,
and $\Delta^i$ is $\sG$-measurable,
\begin{equation}
\E[\widehat D^i\mid\sF]
=\frac{d_i}{\delta}\E\bigl[\Delta^iQ^i[W](x,A)\mid\sF\bigr].
\label{eq:one-point-tower}
\end{equation}
Because $\Delta^i$ is conditionally independent of $A^{-i}$ and has conditional mean zero,
$$
\E\bigl[\Delta^iQ^i[W](x,\widetilde a^i,A^{-i})\mid\sF\bigr]=0.
$$
Taylor's formula in agent $i$'s action, together with \eqref{eq:Phi-action-lip}, gives
$$
Q^i[W](x,\widetilde a^i+\delta\Delta^i,A^{-i})
=Q^i[W](x,\widetilde a^i,A^{-i})+\delta\ip{\Phi^i[W](x,\widetilde a^i,A^{-i})}{\Delta^i}+R_i,
$$
where $|R_i|\le L_\Phi(\fm)\delta^2\|\Delta^i\|^2/2$. 

We use \eqref{eq:one-point-tower} and the conditional independence and 
the covariance assumptions to arrive at
$$
\E[\widehat D^i\mid\sF]
=d_i\E\bigl[\Delta^i(\Delta^i)^\top
\Phi^i[W](x,\widetilde a^i,A^{-i})\mid\sF\bigr]+\varepsilon_i,
$$
where
$$
\varepsilon_i:=\frac{d_i}{\delta}\E[\Delta^iR_i\mid\sF]
\qquad \Rightarrow \qquad
\|\varepsilon\|\le \frac{L_\Phi(\fm)}{2} (\sum_{i=1}^Nd_i^2)^{1/2} \ \delta \le K \delta,
$$
for some constant $K$.
Moreover, the gradient $\Phi^i[W](x,\widetilde a^i,A^{-i})$
is measurable with respect to the sigma-field $\sF\vee\sigma(A^{-i})$ 
and conditionally independent of $\Delta^i$.
Since  $d_i\,\E[\Delta^i(\Delta^i)^\top\mid\sF]=I_{d_i}$,
we have
$$
d_i\,\E\bigl[\Delta^i(\Delta^i)^\top \Phi^i[W](x,\widetilde a^i,A^{-i})\mid\sF\bigr]
= \E\bigl[ \Phi^i[W](x,\widetilde a^i,A^{-i}) \mid \sF \bigr].
$$

Since $\widetilde a=(1-\lambda)a$ and $\|\Delta^j\|\le1$,
$\|(\widetilde a^i,A^{-i})-a\|
\le\lambda\max_{u\in\cU}\|u\|+\sqrt N \delta$.
We now apply the Lipschitz estimate \eqref{eq:Phi-action-lip} to each component 
of the gradient to obtain the following estimate
\begin{align*}
\|\E[\widehat D\mid\sF]-\Phi[W](x,a)\|  &\le \Big(
\sum_{i=1}^N \| \E[ \Phi^i[W](x,\widetilde a^i,A^{-i}) -\Phi^i[W](x,a) \mid\sF]\|^2
\Big)^{1/2}  + \|\varepsilon\|\\
& \le \sqrt N L_\Phi(\fm)
\Big(\lambda\max_{u\in\cU}\|u\|+\sqrt N \delta\Big)
+K \delta \le K(\delta+\lambda),
\end{align*}
where in the final inequality we increase $K$ if necessary.

Finally, with $\xi=\widehat D-\E[\widehat D\mid\sF]$,
\begin{equation*}
\E[\|\xi\|^2\mid\sF] \le\E[\|\widehat D\|^2\mid\sF]
\le \|\widetilde V\|^2_{L^\infty} \sum_{i=1}^Nd_i^2 / \delta^2.
\qedhere
\end{equation*}
\end{proof}

\bibliographystyle{abbrvnat}
\bibliography{MARL}

\end{document}